\documentclass{article}

\PassOptionsToPackage{numbers, compress}{natbib}
\usepackage[preprint]{neurips_2025}

\usepackage{multirow}

\usepackage[utf8]{inputenc} % allow utf-8 input
\usepackage[T1]{fontenc}    % use 8-bit T1 fonts
\usepackage{hyperref}       % hyperlinks
\usepackage{url}            % simple URL typesetting
\usepackage{booktabs}       % professional-quality tables
\usepackage{amsfonts}       % blackboard math symbols
\usepackage{nicefrac}       % compact symbols for 1/2, etc.
\usepackage{microtype}      % microtypography
\usepackage{xcolor}         % colors
\usepackage{xpatch}
\usepackage{amsmath,amssymb,mathtools}
\usepackage{graphicx}
\usepackage{amsfonts} % 仍然需要这个

\usepackage{amsthm}  % 如果还没引入
\usepackage{float}
\usepackage{rotating}

\newtheorem{proposition}{Proposition}

\usepackage{amsthm,amsmath,amsfonts} 
\newtheorem{assumption}{Assumption}
\newtheorem{theorem}{Theorem}

\newtheorem{lemma}{Lemma}

\usepackage{xpatch}
\makeatletter
\xapptocmd{\NAT@bibsetnum}{\setlength{\leftmargin}{0pt}\setlength{\itemindent}{\labelwidth}\addtolength{\itemindent}{\labelsep}}{}{}
\makeatother

\title{FAME: Adaptive Functional Attention with Expert Routing for Function-on-Function Regression \thanks{Accepted at NeurIPS 2025}
}

\author{%
	Yifei Gao\\
	Department of Industrial Engineering\\
    Tsinghua University\\
	Beijing, China 100084 \\
	\texttt{gao-yf@mail.tsinghua.edu.cn} \\
	\And
	Yong Chen\\
	Department of Industrial and Systems Engineering\\ University of Iowa\\ Iowa City, USA 52242\\
	\texttt{yong-chen@uiowa.edu} 
	\And 
	Chen Zhang\thanks{Corresponding author}\\
	Department of Industrial Engineering\\
	Tsinghua University\\
	Beijing, China 100084 \\
	\texttt{zhangchen01@tsinghua.edu.cn} \\
}

\begin{document}

	\maketitle

\begin{abstract}
	Functional data play a pivotal role across science and engineering, yet their infinite-dimensional nature makes representation learning challenging. Conventional statistical models depend on pre-chosen basis expansions or kernels, limiting the flexibility of data-driven discovery, while many deep-learning pipelines treat functions as fixed-grid vectors, ignoring inherent continuity. In this paper, we introduce \textbf{F}unctional \textbf{A}ttention with a \textbf{M}ixture-of-\textbf{E}xperts (FAME), an end-to-end, fully data-driven framework for function-on-function regression. FAME forms continuous attention by coupling a bidirectional neural controlled differential equation with MoE-driven vector fields to capture intra-functional continuity, and further fuses change to inter-functional dependencies via multi-head cross attention. Extensive experiments on synthetic and real-world functional-regression benchmarks show that FAME achieves state-of-the-art accuracy, strong robustness to arbitrarily sampled discrete observations of functions.
\end{abstract}

\section{Introduction}
Functional data are samples whose individual elements are random continuous
functions, providing rich temporal–spatial dynamics in domains such as
biology \citep{jiang2020functional}, marketing \citep{jank2011automated},
transportation \citep{lan2023mm}, and meteorology \citep{wen2023}. 
With the growing demand for high-resolution measurements, functional data
analysis (FDA) is attracting increasing attention in both scientific
research and practical applications\citep{wang2023deep,park2022sparse,liu2023dynamic,zhang2022high}.
Among canonical problems in functional data analysis (FDA), \textbf{function-on-function regression (FoFR)} is widely regarded as the most challenging and representative task, as it requires the model to accommodate both infinite-dimensional inputs and outputs.
\citep{sun2018optimal,mutis2025function,hullait2021robust}.
While functional data offer new opportunities, they also bring intertwined challenges:  
\textbf{(1) Intra-functional continuity}: each function lies in infinite-dimensional continuous functional space with features such as local dynamics and global trends\citep{ramsay2007applied};  \textbf{(2) Inter-functional interactions}: different dimensions of the input function can have nonlinear couplings with each other
\citep{fan2015functional};
\textbf{(3) Feature heterogeneity}: different functions may exhibit vastly different properties such as scale, smoothness, or noise, and may even reside in different functional spaces
\citep{febrero2019variable}. 

To address these challenges, existing approaches broadly fall into two methodological categories. The first category comprises classical statistical methods, which manage functional complexity by projecting each function onto finite-dimensional representations. Common techniques include basis expansions, such as B-splines \citep{crambes2009smoothing}, wavelets \citep{zhao2012wavelet}, and functional principal component analysis (FPCA) \citep{hall2007methodology}, as well as kernel-based methods that embed functions into functional reproducing kernel Hilbert spaces via operator-valued kernels \citep{bouche2021nonlinear}. Although these methods facilitate linear or additive modeling in lower-dimensional spaces, their performance heavily depends on predefined bases or kernels and strong smoothness assumptions, limiting their ability to capture intra-functional and inter-functional features prevalent in real-world functional data.  
The second category involves recent deep learning approaches, which offer greater flexibility inspired by advances in natural language processing and computer vision. These methods typically feed discretized curves into deep neural networks. For instance, \citet{shi2024nonlinear} proposes a smooth-kernel neural network tailored to nonlinear functional regression. \citet{yao2021deep} introduces adaptive basis learners that adjust expansions according to specific tasks. Although such models enhance nonlinearity and support end-to-end learning, they still assume all functions lie in discretized regular sampling grids with the same dimension, and treat functional data as finite-dimensional vectors. Consequently, these methods have limited power to model intra-functional continuous dynamics, and cannot be applied to irregularly sampled cases, let alone handle feature heterogeneity. 

To overcome these limitations, we introduce Functional Attention with Mixture-of-Experts (FAME), an end-to-end framework designed expressly for FOFR. Specifically, our contributions are: (1) FAME is the first FoFR model that directly operates on irregularly sampled functional space without relying on predefined basis functions or discretisation grids. Its effectiveness is rigorously established through both theoretical guarantees and extensive experimental validation. (2) FAME includes a novel functional attention mechanism composed of continuous attention via bidirectional  Neural controlled differential equations (NCDEs), which can efficiently capture intra-functional continuity, and multi-head cross attention, which can efficiently model inter-functional interactions.
(3) FAME enhances a mixture-of-experts (MoE) architecture, enabling adaptive modeling of feature heterogeneity, and incorporates an NCDE decoder capable of generating continuous functional outputs at arbitrary query locations, naturally accommodating misaligned target indices.

\section{Related work}
\label{related}
\paragraph{Function-on-function regression}
FoFR has witnessed substantial advances in recent years and is drawing increasing attention from the research community. Existing approaches can be concretely divided into three categories: linear decomposition, reproducing kernel Hilbert space (RKHS), and deep learning methods. Linear approaches employ finite bases such as FPCA, wavelets, or splines to project input and response curves into low-dimensional coordinates, enabling conventional regression techniques \citep{yao2005functional,muller2008functional,luo2016functional,luo2017function,qi2019nonlinear}. Despite computational efficiency, these approaches rely on fixed bases, restricting their ability to capture non-stationary dynamics and complex interactions. RKHS-based methods use operator-valued kernels, providing nonlinear function regression within linear frameworks \citep{lian2007nonlinear,kadri2016operator}. Extensions have incorporated robust losses \citep{lambert2022functional} and optimal estimation strategies \citep{sun2018optimal}, yet kernel selection and scalability remain persistent challenges. Deep learning methods, including functional neural networks (FNNs) \citep{rossi2006theoretical}, FPCA-based neural networks \citep{wang2019multilayer,wang2020non}, and adaptive basis expansions \citep{yao2021deep}, offer data-driven flexibility. However, these techniques typically discretize functions onto fixed grids, ignoring the inherent continuity and heterogeneity of functional data. In summary, current methods rely on either predefined or prematurely discretized feature representation for functional data, limiting their ability to model complex FoFR problems comprehensively.

\paragraph{Neural differential equations}
Neural differential equations provide a powerful framework that combines continuous-time dynamics with the high-capacity function approximation of neural networks, making them particularly effective for handling irregular time series \citep{rubanova2019latent, jia2019neural}. Among these methods, Neural ODEs map discrete data to continuous trajectories and leverage adjoint methods for memory-efficient gradient computation, enabling scalable training for large models \citep{ardizzone2018analyzing}. Building on this concept, Neural CDEs \citep{kidger2020neural} extend the approach by capturing the continuous-time dynamics of RNNs through a learned control function for hidden states. This strategy has shown superior performance in modeling irregular time series compared to Neural ODEs or RNNs, particularly in offline prediction tasks \citep{bellot2021policy}.
 Recent work also emphasizes their usefulness in enhancing stability and generalization by injecting noise directly into the training process \citep{oh2024stable}. However, Neural CDEs operate in a strictly causal fashion, building hidden-state dynamics solely from past observations, whereas FoFR can exploit the full input function to capture global dependencies and non-causal relationships.

\paragraph{Attention mechanism for feature representation} The attention mechanism was initially proposed to mitigate the limitations of encoding sequences into fixed-length vectors in RNN-based models \citep{bahdanau2014neural}. \citet{xu2015show} introduces soft alignment during decoding, enabling better handling of long sequences. This concept evolves through variants like global and local attention, and reaches a milestone with the Transformer architecture \citep{vaswani2017attention}, which entirely replaces recurrence with self-attention, improving both modeling power and computational efficiency. Building on this foundation, numerous attention-based models have emerged for time series modeling. For example, iTransformer\citep{liu2023itransformer} reorganises the input so that time steps serve as channels and variables act as tokens, allowing self-attention to capture long-range temporal dependencies more effectively, whereas CrossFormer\citep{zhang2023crossformer} employs a frequency-aware, cross-scale attention mechanism to model long- and short-term patterns simultaneously. However, these attention mechanisms are designed for discrete sequences defined on fixed grids. They do not provide a native framework for continuous-space mapping of functions.

\section {Function-on-function regression}
We consider an unknown operator \(\mathcal{T}:\mathcal{X}\!\to\!\mathcal{Y}\) that maps a
\(d\)-tuple of input functions \(X=(X^{(1)},\dots,X^{(d)})\) to an
\(m\)-tuple of output functions \(Y=(Y^{(1)},\dots,Y^{(m)})\).
For every input channel \(j\in\{1,\dots,d\}\) the function
\(X^{(j)}:[0,T_j]\!\to\!\mathbb{R}\) is assumed to have finite
\(1\)-variation; we denote the space by
\[
\mathcal{V}_j=\bigl\{\,X^{(j)}:[0,T_j]\!\to\!\mathbb{R}
\mid\|X^{(j)}\|_{1\mathrm{-var}}<\infty\bigr\},
\quad
\mathcal{X}= \mathcal{V}_1\times\cdots\times\mathcal{V}_d .
\]
For each output index \(\zeta\in\{1,\dots,m\}\), the output function
\(
Y^{(\zeta)}:[0,S_{\zeta}]\to\mathbb{R}
\)
is continuous and equipped with the supremum norm
\(
\|Y^{(\zeta)}\|_\infty=\sup_{s\in[0,S_{\zeta}]}|Y^{(\zeta)}(s)|.
\)
We set
\[
\mathcal{W}_{\zeta}=C\bigl([0,S_{\zeta}],\mathbb{R}\bigr),
\quad
\mathcal{Y}= \mathcal{W}_{1}\times\cdots\times\mathcal{W}_{m},
\quad
\|Y\|_\infty=\max_{\zeta}\|Y^{(\zeta)}\|_\infty .
\]

In practice, for each sample \(i\) we observe
\(
X_i=(X_i^{(1)},\dots,X_i^{(d)})\in\mathcal{X},\quad
Y_i=(Y_i^{(1)},\dots,Y_i^{(m)})\in\mathcal{Y},
\)
but only at irregular time points
\(\{(t_{i,\ell},x_{i,\ell})\}_{\ell=1}^{\Lambda_i}\) for each \(X_i^{(j)}\)
and
\(\{(s_{i,r},y_{i,r})\}_{r=1}^{\Gamma_i}\) for each \(Y_i^{(\zeta)}\).
Given a dataset \(\{(X_i,Y_i)\}_{i=1}^{N_s}\subset\mathcal{X}\times\mathcal{Y}\),
our goal is to learn a finite-parameter approximation
\(\mathcal{T}_{\theta}\) that recovers
\(\mathcal{T}\) from these discretely and non-uniformly observed functions.

\section{Functional attention with a mixture-of-experts (FAME)}
\label{sec:overview}
\begin{figure*}[t]
	\begin{center}
		\centerline{\includegraphics[width=\columnwidth]{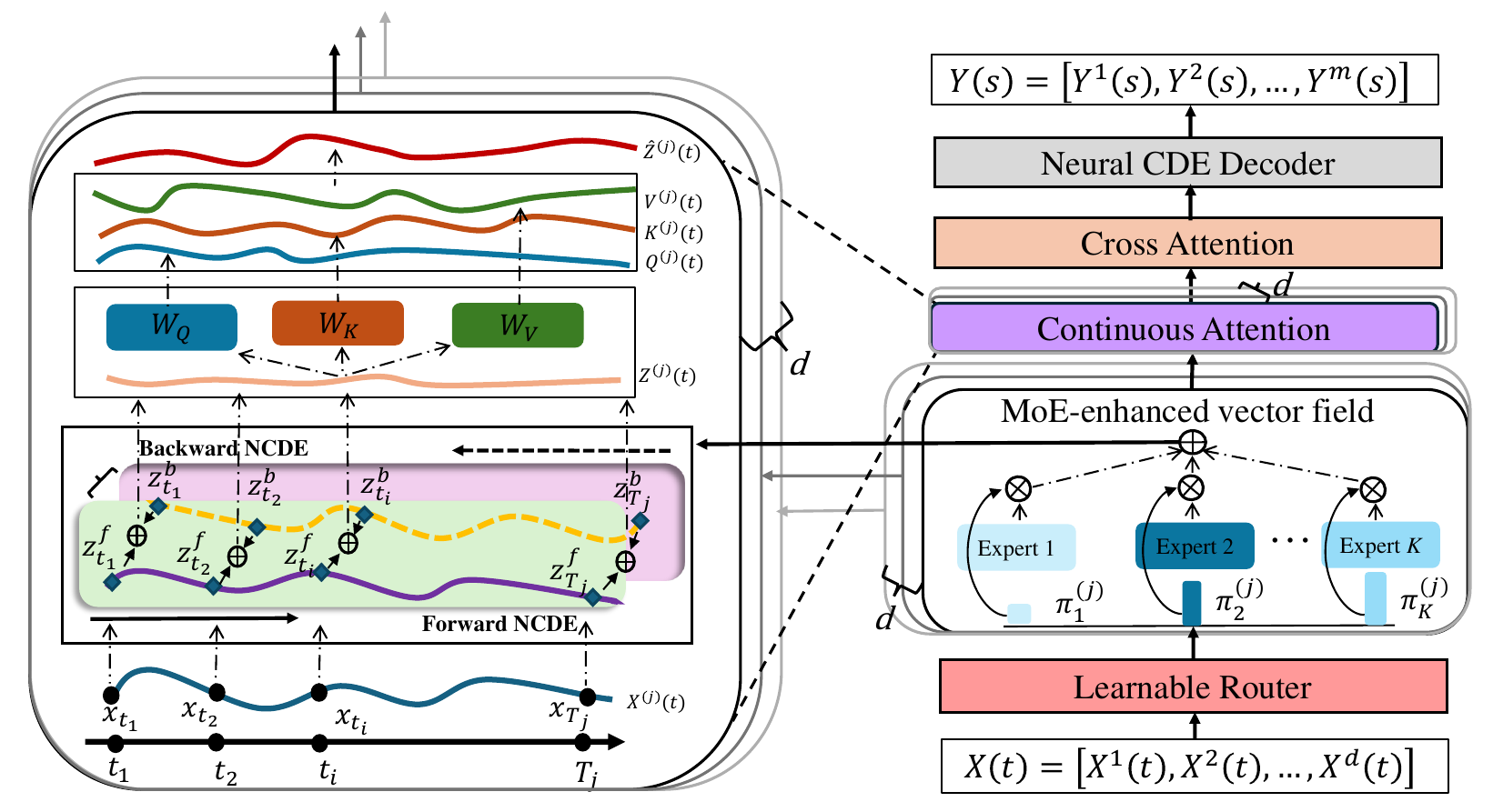}}
		\caption{Architecture of the FAME.}
		\label{FAME}	
	\end{center}
	\vskip -0.3in
\end{figure*}

In this section, we describe how FAME addresses the core challenges of FoFR by designing a custom functional attention mechanism. Figure~\ref{FAME} presents a high-level overview of the overall architecture of FAME.  
To account for the truly infinite-dimensional nature of each function \(X^{(j)}\), we introduce in Section~\ref{sec:attention-block} a continuous attention block built on a bidirectional NCDE that integrates both forward and backward over $t$. This construction produces smoothly varying Query, Key, and Value trajectories, ensuring that every query instant attends to both past and future contexts and remains insensitive to the discrete sampling along the input function.  
To capture heterogeneity across functions, Section~\ref{sec:moe-cde} enhances this continuous attention block by replacing its single vector field with a mixture of \(K\) expert fields (MoE), each specialised for distinct scales, amplitudes, or noise regimes. A learnable router allocates a weight vector over the \(K\) experts, thereby boosting the model’s ability to capture feature heterogeneity, removing the need to instantiate a separate encoder for each function, and enabling smooth generalisation to unseen features.
 Section~\ref{sec:cross-att} then introduces a cross attention fusion block that integrates these per-function representations into a unified global context, propagating information across functions and modelling their continuously evolving, nonlinear couplings, before decoding via a NCDE head to produce the final continuous output.

%=========================== 3.3 ======================================

\subsection{Capturing intra-functional continuity  through continuous attention}
\label{sec:attention-block}

%-----------------------------------------------------------------
\paragraph{Controlled differential equations.}

The core motivation of our method is to construct attention mechanisms that operate directly on continuous functional data. We adopt the framework of neural controlled differential equations (NCDEs), whose integrals are interpreted in the Young sense because every input function has finite $1$-variation. In a controlled differential equation (CDE)\cite{kidger2020neural}, an external driver continuously influences the internal dynamics, providing a principled route for learning from irregularly sampled paths. Formally, the latent trajectory is defined as the solution of:  
\begin{equation}\label{eq:singleCDE}
	z(t) = z(t_0 )
	+ \int_{t_0}^{t} f_\theta\!\bigl(z(u)\bigr)\,dX(u),
	\qquad t\in(t_0,T_j],
\end{equation}
where the initial state \(z(t_0)=\xi_\theta\bigl(X(t_0),t_0\bigr)\) is learnable, \(z(t)\) encodes an evolving summary of the input path, and \(f_\theta\) governs how this summary changes in response to the signal \(X\). 

%-----------------------------------------------------------------
\paragraph{Bidirectional NCDE.}

A one-sided CDE only observes the driving signal on \([t_0,t]\), so its latent state lacks access to the future segment \((t,T_j]\). To capture the essential global structure of functional data, we propose a bidirectional NCDE that augments the standard forward integration with a complementary backward pass.
For each function \(X^{(j)}\), we solve a pair of controlled differential equations:
\begin{equation}\label{eq:bidirCDE}
	\begin{aligned}
		Z_{\text{fwd}}^{(j)}(t)
		&= Z^{(j)}(t_0)
		+\int_{t_0}^{t}
		f_{\theta^{\text{fwd}}_j}\!\bigl(Z_{\text{fwd}}^{(j)}(\tau)\bigr)
		\,dX^{(j)}(\tau), \\
		Z_{\text{bwd}}^{(j)}(t)
		&= Z^{(j)}(T_j)
		-\int_{t}^{T_j}
		f_{\theta^{\text{bwd}}_j}\!\bigl(Z_{\text{bwd}}^{(j)}(\tau)\bigr)
		\,d\tilde X^{(j)}(\tau),
	\end{aligned}
	\qquad t\in[t_0,T_j],
\end{equation}
where \(f_{\theta^{\text{fwd}}_j}\) and \(f_{\theta^{\text{bwd}}_j}\) are direction-specific vector fields, and \(\tilde X^{(j)}\) denotes the reversed input function \(\tilde X^{(j)}(t) = X^{(j)}(T_j - t)\).

\begin{assumption}[Directional Lipschitz regularity]\label{ass:bidir-lip}
	Each vector field is globally Lipschitz: that is,
	\(f_{\theta^{\text{fwd}}_j}\) is \(L^{\text{fwd}}_j\)-Lipschitz and
	\(f_{\theta^{\text{bwd}}_j}\) is \(L^{\text{bwd}}_j\)-Lipschitz.
\end{assumption}

\begin{theorem}[Existence and uniqueness of the Bi-NCDE]\label{thm:bidir-exist}
	Under Assumption~\ref{ass:bidir-lip}, both the forward and backward CDEs in \eqref{eq:bidirCDE} admit unique solutions on \([t_0,T_j]\). Moreover, letting
	\(L_j=\max\{L^{\text{fwd}}_j,\,L^{\text{bwd}}_j\}\), the latent paths satisfy
	\[\label{eq:bidirStab}
	\bigl\|Z^{(j)} - \tilde Z^{(j)}\bigr\|_\infty
	\;\le\;
	e^{L_j(T_j-t_0)}\,\bigl\|X^{(j)} - \tilde X^{(j)}\bigr\|_{1\text{-var}} .
	\]
\end{theorem}

Then, the final latent representation is defined as \( Z^{(j)}(t) = [Z_{\text{fwd}}^{(j)}(t),\, Z_{\text{bwd}}^{(j)}(t)] \in \mathbb{R}^{2h} \), 
combining local dynamics with global context.

%-----------------------------------------------------------------
\paragraph{Continuous attention.}

To extend attention from discrete sequences to continuous functions, we apply three bounded matrices 
\(
W_Q,\,W_K,\,W_V \;\in\;\mathbb{R}^{d_f\times 2h}
\)
to \(Z^{(j)}(t)\), thereby obtaining continuous Query, Key, and Value functions:
\begin{equation}
	Q^{(j)}(t)=W_Q\,Z^{(j)}(t),\quad
	K^{(j)}(t)=W_K\,Z^{(j)}(t),\quad
	V^{(j)}(t)=W_V\,Z^{(j)}(t),
	\quad t\in[t_0,T_j].
\end{equation}
Because bounded linear maps preserve Young-integrability and
Lipschitz regularity, the ensuing attention operations are well defined
over \([t_0,T_j]\) and reserve the continuous structure of the data.
For any \((t,\tau)\), we compute the attention score and normalized weight  
\begin{equation}\label{eq:kernel-softmax}
	\alpha^{(j)}(t,\tau)=
	\frac{\langle Q^{(j)}(t),K^{(j)}(\tau)\rangle}{\sqrt{d_f}},
	\qquad
	\hat\alpha^{(j)}(t,\tau)=
	\frac{e^{\alpha^{(j)}(t,\tau)}}{\displaystyle
		\int_{t_0}^{T_j} e^{\alpha^{(j)}(t,u)}\,du},
\end{equation}

and form the attended representation  
\begin{equation}\label{eq:zhat}
	\hat Z^{(j)}(t)=
	\int_{t_0}^{T_j}\hat\alpha^{(j)}(t,\tau)\,V^{(j)}(\tau)\,d\tau,
	\qquad t\in[t_0,T_j].
\end{equation}
\begin{assumption}[Normalization  regularity]\label{ass:softmax}
	There exists \(\gamma>0\) such that
	\(\int_{t_0}^{T_j} e^{\alpha^{(j)}(t,s)}\,ds\ge\gamma\)
	for all \(t\) and admissible inputs.
	In practice we ensure this by a temperature parameter
	\(\tau^{\mathrm{temp}}\ge\tau^{\mathrm{temp}}_0>0\) and log-sum-exp
	stabilisation.
\end{assumption}

\begin{proposition}[Lipschitz stability]\label{prop:lip}
	Under Assumptions \ref{ass:bidir-lip} and \ref{ass:softmax},
	\[
	\bigl\|\hat Z^{(j)}-\hat{\tilde Z}^{(j)}\bigr\|_\infty
	\le
	\frac{\|W_V\|}{\gamma}\,
	e^{L_j(T_j-t_0)}L_j(T_j-t_0)\,
	\bigl\|X^{(j)}-\tilde X^{(j)}\bigr\|_{1\text{-var}}.
	\]
	Denote the constant on the right by \(L_{\mathrm{single}}\) for later
	reference.
\end{proposition}
\begin{proposition}[Universal approximation]\label{prop:ua}
	Let \(k\in C([t_0,T_j]^2)\) and \(\varepsilon>0\).
	If \(d_f\ge h\) and the width of \(f_\theta\) is sufficiently large,
	then there exist parameters \((f_\theta,W_Q,W_K,W_V)\) with
	\(\bigl\|\alpha^{(j)}-k\bigr\|_\infty<\varepsilon\).
	Hence \(X^{(j)}\mapsto\hat Z^{(j)}\) can approximate any bounded
	continuous operator on functions, independent of the sampling grid;
\end{proposition}

Proposition \ref{prop:lip} confirms the Lipschitz stability of our continuous attention mechanism: a small perturbation of the input function produces only a proportionally small change in its attended representation. Moreover, Proposition \ref{prop:ua} shows that, with sufficient width, the continuous attention can approximate any bounded continuous kernel and thus emulate arbitrary integral operators in FoFR. Consequently, by embedding continuous attention within a bidirectional NCDE, our model not only captures both local dynamics and global context but also can realize any mapping in the function space. More details are provided in Appendix~\ref{app:theory}.

%=========================== 3.4 ======================================
\subsection{Encoding feature-wise heterogeneity with MoE-Driven Bi-NCDE}
\label{sec:moe-cde}

Real-world functional data exhibit evident feature heterogeneity, as each function may vary on its own scale, frequency, and noise level. A single vector field can scarcely accommodate such diversity. To address this, we assign each function its own MoE vector field. Specifically, a shared router computes, once per function, a fixed convex combination of $K$ specialist fields, which then governs the bidirectional NCDE dynamics for that function.

%-----------------------------------------------------------------
\paragraph{Function-wise router and expert fields.}
Let $\{f_{\theta_k}:\mathbb R^{h}\!\to\!\mathbb R^{h}\}_{k=1}^{K}$ be
$K$ expert vector fields.
For each input function $X^{(j)}$ we compress its functional information with a
lightweight 1-D convolution followed by global average pooling,
obtaining a summary vector
$
s^{(j)}=\mathrm{ConvPool}\!\bigl(X^{(j)}\bigr)\in\mathbb R^{h_0}.
$ The router $g_\phi:\mathbb R^{h_0}\!\to\!\mathbb R^{K}$ assigns softmax
weights $\pi^{(j)}=\mathrm{softmax}\bigl(g_\phi(s^{(j)})\bigr)\in\Delta^{K-1}$
to the $K$ specialist fields, yielding the MoE vector field
\begin{equation}\label{eq:moe_static_field}
	f_{\Theta}^{(j)}(z)=\sum_{k=1}^{K}\pi_k^{(j)}\,f_{\theta_k}(z),\qquad
	\Theta=\{\theta_1,\dots,\theta_K,\phi\}.
\end{equation}

%-----------------------------------------------------------------
\paragraph{Directional experts with a shared router.}
To preserve the bidirectional representation we instantiate disjoint
expert sets $\{f_{\theta_k^{\text{fwd}}}\}$ and
$\{f_{\theta_k^{\text{bwd}}}\}$, yet reuse the same gates
$\pi^{(j)}$ in both directions:
\[
f_{\Theta^{\text{fwd}}}^{(j)}(z)=
\sum_{k}\pi_k^{(j)}\,f_{\theta_k^{\text{fwd}}}(z),\qquad
f_{\Theta^{\text{bwd}}}^{(j)}(z)=
\sum_{k}\pi_k^{(j)}\,f_{\theta_k^{\text{bwd}}}(z).
\]

%-----------------------------------------------------------------
\begin{assumption}[MoE regularity]\label{ass:moe-static}
	Each expert $f_{\theta_k^{\text{fwd}}}$ and $f_{\theta_k^{\text{bwd}}}$
	is globally \(L_k^{\text{fwd}}\)- and \(L_k^{\text{bwd}}\)-Lipschitz, respectively, and bounded by \(B_k^{\text{fwd}}\) and \(B_k^{\text{bwd}}\).
	The router $g_\phi$ is bounded on the summary space: 
	$\|g_\phi(s)\|_\infty\le B_g$ for all $s\in\mathbb R^{h_0}$.
\end{assumption}

\begin{lemma}[Mixed-field Lipschitz bound]\label{lem:moeLip-static}
	Under Assumption~\ref{ass:moe-static}, the mixed fields
	$f_{\Theta^{\text{fwd}}}^{(j)}$ and $f_{\Theta^{\text{bwd}}}^{(j)}$ are
	globally Lipschitz with
	\[
	L_{\text{mix}}
	=\max\!\Bigl\{\sum_{k}\pi_k^{(j)}L_k^{\text{fwd}},\;
	\sum_{k}\pi_k^{(j)}L_k^{\text{bwd}}\Bigr\}
	\;\le\;\max_{k}\bigl(L_k^{\text{fwd}},L_k^{\text{bwd}}\bigr).
	\]
	Consequently, the bidirectional NCDE driven by
	\eqref{eq:moe_static_field} satisfies the existence, uniqueness, and
	stability properties of Theorem~\ref{thm:bidir-exist}.
\end{lemma}

%-----------------------------------------------------------------
\paragraph{MoE-Driven Bi-NCDE.}
Replacing $f_{\theta^{\text{fwd}}_j}$ and $f_{\theta^{\text{bwd}}_j}$ in
\eqref{eq:bidirCDE} with
$f_{\Theta^{\text{fwd}}}^{(j)}$ and $f_{\Theta^{\text{bwd}}}^{(j)}$
yields the MoE-Driven Bi-NCDE, which enables each input function to employ a custom vector field mixture and thereby model input functions
with markedly different characteristics.  

\begin{theorem}[Well-posed MoE-Driven Bi-NCDE]\label{thm:moeExist-static}
	The MoE augmentation admits unique solutions, and the stability
	bound~\eqref{eq:bidirStab} holds with
	$L_j=\max\{L_{\text{mix}}^{\text{fwd}},L_{\text{mix}}^{\text{bwd}}\}$.
\end{theorem}

Thus the MoE design endows the encoder with function-wise
expressiveness while retaining all continuity, sampling invariance, and
Lipschitz stability guarantees of the base architecture.

\subsection{Integrating inter-functional interactions via cross attention}
\label{sec:cross-att}

Having obtained a function-wise representation \(\hat Z^{(j)}\) for each input function, we now capture the inter-functional dependencies that single-channel encoders miss. We realise this with a multi-head cross attention that, at every time point \(t\), allows each function to attend to the concurrent states of all other functions.

%-----------------------------------------------------------------
\paragraph{Cross attention.}
For every head \(p\in\{1,\dots,M\}\) and function \(j\in\{1,\dots,d\}\), shared linear maps  
\(W^{(p)}_{Q},W^{(p)}_{K},W^{(p)}_{V}\in\mathbb R^{d_c\times d_f}\)
project the latent paths into query, key, and value functions:
\begin{equation}\label{eq:crossQKV}
	Q^{(j,p)}(t)=W^{(p)}_{Q}\hat Z^{(j)}(t),\quad
	K^{(j,p)}(t)=W^{(p)}_{K}\hat Z^{(j)}(t),\quad
	V^{(j,p)}(t)=W^{(p)}_{V}\hat Z^{(j)}(t).
\end{equation}
Cross attention weights are defined by
\begin{equation}\label{eq:crossScoreWeight}
	\hat\beta^{(j,\ell,p)}(t)=
	\frac{\exp\{\langle Q^{(j,p)}(t),K^{(\ell,p)}(t)\rangle/\sqrt{d_c}\}}
	{\displaystyle\sum_{r=1}^{d}
		\exp\{\langle Q^{(j,p)}(t),K^{(r,p)}(t)\rangle/\sqrt{d_c}\}},
	\qquad \ell\in\{1,\dots,d\}.
\end{equation}
and each head outputs
\begin{equation}\label{eq:crossHeadOut}
	H^{(j,p)}(t)=
	\sum_{\ell=1}^{d}\hat\beta^{(j,\ell,p)}(t)\,V^{(\ell,p)}(t).
\end{equation}

Concatenating the $M$ heads and applying
\(W_O\in\mathbb R^{2h\times M d_c}\) gives
\begin{equation}\label{eq:crossOut}
	H^{(j)}(t)=
	W_O\bigl[H^{(j,1)}(t)\|\dots\|H^{(j,M)}(t)\bigr],
	\qquad t\in[t_0,T_j].
\end{equation}
Here \(W_O\) projects the concatenated representation back to the
original \(2h\)-dimensional latent space,  keeping the hidden width uniform throughout
the network.

\begin{assumption}[Cross attention Lipschitz]\label{ass:crossLip}
	For every head \(p\in\{1,\dots,M\}\) the projection matrices satisfy 
	\(\|W^{(p)}_{Q}\|_{\mathrm{op}},
	\|W^{(p)}_{K}\|_{\mathrm{op}},
	\|W^{(p)}_{V}\|_{\mathrm{op}},
	\|W_O\|_{\mathrm{op}}\le M_{mat}\)
	for some constant \(M_{mat}>0\). 
	Under this condition, the cross attention map \(\hat Z\mapsto H\) is globally \(L_{\text{cross}}\)-Lipschitz with \(L_{\text{cross}}\le (M_{mat})^{2}/\sqrt{d_c}\) with respect to the \(1\)-variation norm on function space; see Appendix~\ref{app:lip} for the derivation.
\end{assumption}
%-----------------------------------------------------------------
\paragraph{CDE decoder for functional targets.}
To produce continuous functional outputs, we feed the stacked path  
\(H(t)=\bigl[H^{(1)}(t)\|\dots\|H^{(d)}(t)\bigr]\in\mathbb R^{2hd}\)
into a shallow neural CDE:
\begin{equation}\label{eq:cdeHeadFinal}
	\hat Y^{(\zeta)}(s)=y^{(\zeta)}(s_0)
	+ \int_{0}^{s} f_{\psi}\!\bigl(\hat Y^{(\zeta)}(u)\bigr)\,dH(u),
	\qquad s\in[0,S_\zeta],
\end{equation}
where \(f_{\psi}:\mathbb R^{m}\!\to\!\mathbb R^{m}\) is a learnable vector field.
Because the solution of \eqref{eq:cdeHeadFinal} can be evaluated at \emph{any}
index \(s\), the decoder naturally accommodates misaligned or irregular
target grids.

\begin{assumption}[Decoder Lipschitz regularity]\label{ass:decLip}
	The decoder field \(f_{\psi}\) is globally \(L_{\psi}\)-Lipschitz and bounded:
	\(\|f_{\psi}\|_{\infty}\le B_{\psi}\).
\end{assumption}

\begin{theorem}[Well-Posedness and Lipschitz stability of the decoder]\label{thm:decExist}
	Under Assumptions~\ref{ass:crossLip}–\ref{ass:decLip}, the encoder output path \(H\) has finite \(1\)-variation, and
	\eqref{eq:cdeHeadFinal} admits a unique solution
	\(\hat Y^{(\zeta)}\in C\!\bigl([0,S_\zeta],\mathbb R\bigr)\) with
	\[
	\|\hat Y^{(\zeta)}-\tilde Y^{(\zeta)}\|_{\infty}
	\le e^{L_{\psi}S_\zeta}L_{\psi}S_\zeta\,\|H-\tilde H\|_{1\text{-var}}\! .
	\]
\end{theorem}
Hence, for each output index \(\zeta\), the decoder
\(\mathcal D_{\psi}:H\mapsto\hat Y^{(\zeta)}\)
is a bounded, continuous operator from
\(\bigl(C([0,S_\zeta],\mathbb R^{2hd}),\|\cdot\|_{1\text{-var}}\bigr)\)
to the Banach space
\(\bigl(C([0,S_\zeta],\mathbb R),\|\cdot\|_{\infty}\bigr)\).

\vspace{2pt}
\noindent\textbf{End-to-end Lipschitz bound and generalization.}
Composing the encoder  
\(\mathcal E_{\theta}:X\mapsto H\) (continuous and cross attention) with the decoder  
\(\mathcal D_{\psi}:H\mapsto\hat Y\) yields the operator  
\(\mathcal T_{\Theta}=\mathcal D_{\psi}\circ\mathcal E_{\theta}\).
Denote \(S_{\max}=\max_{\zeta}S_\zeta\).  
Under the preceding assumptions, \(\mathcal T_{\Theta}\) is globally Lipschitz with
\[
L_{*}=L_{\text{enc}}\,e^{L_{\psi}S_{\max}}L_{\psi}S_{\max},
\qquad
L_{\text{enc}}=L_{\text{cross}}\,L_{\text{single}}.
\]
This single constant summarizes the theory behind FAME: the model is provably
stable, invariant to sampling grids, and enough
to represent any continuous operator on functions.  
Moreover, the hypothesis class has Rademacher complexity
\(\mathcal O\bigl(L_{*}/\sqrt{N}\bigr)\), supplying a quantitative generalization
guarantee for FoFR.  
Further details and auxiliary lemmas are provided in Appendix~\ref{app:theory}.

\section{Experiments}
To evaluate the proposed model, we benchmark FAME against a wide range of state-of-the-art methods for FoFR (see Section~\ref{related}).  
For basis--expansion pipelines \citep{wang2016functional}, we pair two orthonormal systems---B-splines and Fourier functions---with four classical regressors: Ordinary Linear, Ridge, Lasso, and Elastic Net.  
Beyond these eight variants, we include functional principal component analysis (FPCA) \citep{hall2007methodology}, operator-valued kernel regression \citep{kadri2016operator}, Gaussian processes \citep{shi2011gaussian}, and the functional neural network (FNN) \citep{thind2023deep} as stronger baselines. 
We evaluate FAME and the baselines on both synthetic datasets and several real-world datasets, using mean-squared error (MSE) as the primary evaluation metric.

\subsection{Datasets}\label{sec:datasets}

\paragraph{Synthetic data.}
We generate functional data using Gaussian‐process trajectories because GP paths have finite \(1\)-variation almost surely, hence \(X_i\in\mathcal X\).  Specifically, we draw  
\(X_i^{(j)}(t)\sim\mathcal{GP}\bigl(0,K(t,t')\bigr)\) for \(j=1,\dots,d\) and \(t\in[0,1]\),  
and define the response by  
\(Y_i(t)=\sin\!\bigl(\sum_{j=1}^{d}X_i^{(j)}(t)\bigr)+\lambda\,\varepsilon_i(t)\),  
with \(\varepsilon_i(t)\stackrel{\text{i.i.d.}}{\sim}\mathcal N(0,\sigma^2)\) and noise level \(\lambda\ge0\). Input and output are sampled on (possibly distinct) irregular grids
\(\{t_{i,\ell}\}_{\ell=1}^{\Lambda_i}\) and
\(\{s_{i,r}\}_{r=1}^{\Gamma_i}\).
We set \(d=3,\;m=1,\;\Lambda_i=\Gamma_i=20\) and \(N_s=200\) unless
otherwise specified. In particular, we evaluate FAME under seven controlled settings. \textbf{Case 1} fixes the input and output grids at \( \Lambda_i=\Gamma_i=20 \), perfectly aligned across all samples. \textbf{Case 2} increases the shared resolution to \( \Lambda_i=\Gamma_i=50 \), while preserving grid alignment. \textbf{Case 3} keeps the output grid at \( \Gamma_i=20 \) but samples the input resolution independently as \( \Lambda_i\in\{10,20,50\} \) for each sample, thereby introducing variable input granularity. Throughout Cases 1--3 we vary the dataset size \( N_s\in\{100,200,500\} \) to disentangle the effects of resolution and sample size.  
Moving to stress tests, \textbf{Case 4} introduces feature heterogeneity by assigning each input function a distinct RBF kernel width 
\(\sigma_j\in\{0.2, 0.3, 0.5\}\), rather than the fixed value of \(0.3\) used in the other cases. \textbf{Case 5} probes noise robustness by adding independent Gaussian perturbations of noise level \( \lambda\in\{0.1,0.2,0.3\} \) to every output. \textbf{Case 6} tests multivariate prediction by extending the target to the two-dimensional map \( [\sin(\sum_j X_i^{(j)}(t)),\;\cos(\sum_j X_i^{(j)}(t))]^{\!\top} \). \textbf{Case 7} stresses scalability by increasing the number of input functions to \( d\in\{5,10\} \) while keeping all other settings unchanged. Finally, to isolate the impact of the hyperparameter \(K\) under high dimensionality and heterogeneity, \textbf{Case 8} fixes the number of input functions at \(d=10\), generates inputs with RBF kernels of widths \(\sigma_j\in\{0.2,0.3,0.5\}\), and varies \(K\in\{1,2,3,5,8\}\).

\paragraph{Real-world data.}
We evaluate the same set of models on three public datasets: 1) 
\emph{Hawaii Ocean}, which contains five hydrographic depth profiles—temperature, salinity, oxygen, chloropigment, and density—among which different variables are treated as regression targets in turn, with the remaining serving as input functions; 
2) \emph{Human3.6M}, a human motion capture dataset consisting of 3-D joint trajectories, where we define three action-specific regression tasks (Walking, Eating, and Sitting); and 3)
\emph{ETT-small}, a monthly electricity-transformer dataset used to forecast oil temperature from transformer load curves.
Full preprocessing details and task definitions are provided in Appendix~\ref{experiment_2}.

\subsection{Experimental Results}
\begin{table*}
	\setlength{\tabcolsep}{5pt}
    \caption{Average test MSE for different methods in regression.  Detailed results (mean $\pm$ standard deviation) are provided in Appendix~\ref{experiment_2}. The best MSE for each case is highlighted in bold.}

	\label{tab:1}
	\begin{small}
		\begin{center}
			\begin{tabular}{llccccccccc}
				\toprule
				\multirow{2}{*}{Model} & \multirow{2}{*}{} &
				\multicolumn{3}{c}{Case 1} & \multicolumn{3}{c}{Case 2} & \multicolumn{3}{c}{Case 3} \\
				\cmidrule(lr){3-5} \cmidrule(lr){6-8} \cmidrule(lr){9-11}
				& & 100 & 200 & 500 & 100 & 200 & 500 & 100 & 200 & 500 \\
				\midrule
				\multirow{4}{*}{B-spline} & Linear      & 0.4720 & 0.3947 & 0.3123 & 0.4135 & 0.3412 & 0.2822 & 0.4958 & 0.4002 & 0.3260 \\
				& Ridge       & 0.4264 & 0.3893 & 0.3117 & 0.3960 & 0.3393 & 0.2817 & 0.3869 & 0.3740 & 0.3222 \\
				& Lasso       & 0.4098 & 0.3830 & 0.3052 & 0.3856 & 0.3351 & 0.2766 & 0.4132 & 0.3584 & 0.3188 \\
				& Elastic Net & 0.4510 & 0.3650 & 0.2874 & 0.3725 & 0.3152 & 0.2583 & 0.3986 & 0.3618 & 0.3190 \\
				\multirow{4}{*}{Fourier} & Linear      & 0.5002 & 0.4092 & 0.3224 & 0.4923 & 0.3636 & 0.2896 & 0.4762 & 0.3921 & 0.3547 \\
				& Ridge       & 0.4255 & 0.3780 & 0.3149 & 0.3616 & 0.3343 & 0.2841 & 0.3755 & 0.3568 & 0.3328 \\
				& Lasso       & 0.3550 & 0.3493 & 0.3135 & 0.3361 & 0.3280 & 0.3001 & 0.3808 & 0.3560 & 0.3247 \\
				& Elastic Net & 0.3540 & 0.3325 & 0.2914 & 0.3167 & 0.3070 & 0.2720 & 0.3737 & 0.3496 & 0.3366 \\
				\cmidrule(lr){0-1}
				\multicolumn{2}{c}{FPCA}             & 0.3717 & 0.3554 & 0.3200 & 0.2890 & 0.2733 & 0.2624 & 0.3812 & 0.3563 & 0.3295 \\
				\multicolumn{2}{c}{Kernel Method}    & 0.2441 & 0.1728 & 0.1058 & 0.1741 & 0.0923 & 0.0700 & 0.2654 & 0.2445 & 0.1485 \\
				\multicolumn{2}{c}{Gaussian Process} & 0.3405 & 0.2941 & 0.2036 & 0.3905 & 0.3917 & 0.2588 & 0.3031 & 0.2926 & 0.3945 \\
				\multicolumn{2}{c}{FNN}              & 0.3123 & 0.2083 & 0.1013 & 0.1941 & 0.1142 & 0.0811 & 0.3678 & 0.3571 & 0.1366 \\
				\cmidrule(lr){0-1}
				% ---- New rows from ablation ----
				\multicolumn{2}{c}{FAME w/o Bi-dir}  & 0.1832 & 0.0812 & 0.0654 & 0.1530 & 0.0528 & 0.0355 & 0.1919 & 0.0813 & 0.0368 \\
				\multicolumn{2}{c}{FAME w/o MoE}     & 0.1870 & 0.0828 & 0.0663 & 0.1578 & 0.0538 & 0.0362 & 0.1972 & 0.0856 & 0.0374 \\
				\multicolumn{2}{c}{FAME w/o Cross-attn} & 0.1902 & 0.0815 & 0.0668 & 0.1602 & 0.0544 & 0.0375 & 0.1997 & 0.0879 & 0.0381 \\
				% --------------------------------
				\multicolumn{2}{c}{\textbf{FAME}}    & \textbf{0.1806} & \textbf{0.0783} & \textbf{0.0635}
				& \textbf{0.1532} & \textbf{0.0511} & \textbf{0.0342}
				& \textbf{0.1954} & \textbf{0.0796} & \textbf{0.0352} \\
				\bottomrule
			\end{tabular}
		\end{center}
	\end{small}
\end{table*}
\begin{table*}
	\caption{
		Average test set MSE in regression under simulation. Basis Expansion (best) shows the best result among the 8 basis expansion methods presented in Table 1.}
	\label{tab:2}
	\begin{center}
		\begin{tabular}{lcccccccc}
			\toprule
			\multirow{2}{*}{Model} & \multirow{2}{*}{case 4} & \multicolumn{3}{c}{case 5 } & \multirow{2}{*}{case 6} & \multicolumn{2}{c}{case 7 } & \multirow{2}{*}{case 8} \\
			\cmidrule(lr){3-5} \cmidrule(lr){7-8}
			& & 0.1 & 0.2 & 0.3 & & 5 & 10 & \\
			\midrule
			Basis Expansion(best) & 0.3610 & 0.3665 & 0.3944 & 0.4419 & 0.3669 & 0.4501 & 0.4705 & 0.5204\\
			FPCA                    & 0.3802 & 0.3879 & 0.3956 & 0.4570 & 0.3844 & 0.5284 & 0.5573 & 0.5737 \\
			Kernel Method           & 0.1928 & 0.1440 & 0.1919 & 0.2435 & 0.2022 & 0.4659 & 0.5236 & 0.5895 \\
			Gaussian Process        & 0.2302 & 0.3079 & 0.3356 & 0.3830 & 0.3434 & 0.4120 & 0.4498 & 0.4762 \\
			FNN                     & 0.2102 & 0.1879 & 0.2250 & 0.3438 & 0.1744 & 0.4801 & 0.5164 & 0.5384 \\
			FAME                    & \textbf{0.0798} & \textbf{0.0846} & \textbf{0.1076} & \textbf{0.1420} & \textbf{0.0824} & \textbf{0.2285} & \textbf{0.3330} & \textbf{0.3530} \\
			\bottomrule
		\end{tabular}
	\end{center}
\end{table*}

\begin{table*}
	\caption{Average test set MSE in regression on different real-world datasets.}
	\label{tab:3}
	\begin{center}
		\begin{tabular}{lcccccc}
			\toprule
			\multirow{2}{*}{Model} &
			\multicolumn{2}{c}{Hawaii ocean} &
			\multicolumn{3}{c}{Human3.6M} &
			\multicolumn{1}{c}{ETDataset} \\
			\cmidrule(lr){2-3} \cmidrule(lr){4-6} \cmidrule(lr){7-7}
			& Salinity & Temp & Walking & Eating & Sitting down & Oil Temp \\
			\midrule
			Basis Expansion(best)   & 0.0780 & 0.0014 & 0.0359 & 0.04841 & 0.0122 & 0.0365 \\
			FPCA                   & 0.0865 & 0.0025 & 0.0373 & 0.0099  & 0.0121 & 0.0355 \\
			Kernel Method          & 0.0754 & 0.0025 & 0.0373 & 0.0099  & 0.0121 & 0.0355 \\
			Gaussian Process       & 0.0931 & 0.0022 & 0.0360 & 0.0075  & 0.0107 & 0.0380 \\
			FNN                    & 0.0766 & 0.0020 & 0.0373 & 0.0099  & 0.0121 & 0.0355 \\
			\cmidrule(lr){0-0}
			% ---- New ablation rows (means only) ----
			FAME w/o Bi-dir & 0.0751 & 0.0012 & 0.0327 & 0.0035 & 0.0071 & 0.0264 \\
			FAME w/o MoE             & 0.0759 & 0.0013 & 0.0332 & 0.0038 & 0.0075 & 0.0271 \\
			FAME w/o Cross-attn & 0.0773 & 0.0014 & 0.0344 & 0.0044 & 0.0083 & 0.0286 \\
			% ---------------------------------------
			FAME                  & \textbf{0.0748} & \textbf{0.0012} &
			\textbf{0.0325} & \textbf{0.0034} &
			\textbf{0.0070} & \textbf{0.0262} \\
			\bottomrule
		\end{tabular}
	\end{center}
\end{table*}

\begin{figure}[t]
	\vskip 0.1in
	\begin{center}
		\centerline{\includegraphics[width=\columnwidth]{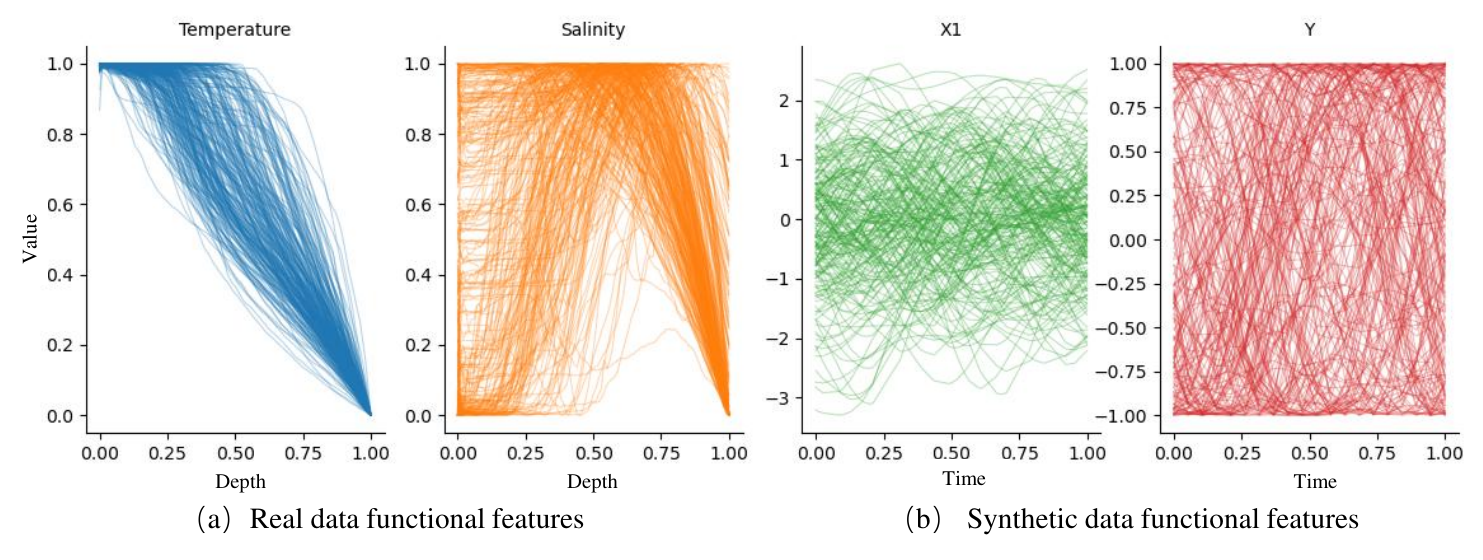}}
		\caption{Visualization of Functional Features.}
		\label{fig:function}	
	\end{center}
	\vskip -0.2in
\end{figure}

Table~\ref{tab:1} summarizes the performance of \textsc{FAME} and various baselines on the synthetic datasets, clearly demonstrating that \textsc{FAME} consistently achieves the lowest test-set MSE. Other methods based on fixed basis expansions, FPCA, Gaussian processes, or neural networks achieve inferior performance, largely due to their limited capacity to adapt to nonlinear functional mappings, sensitivity to irregular sampling points, or challenges in preserving continuous functional structures. In contrast, \textsc{FAME}'s continuous attention mechanism, supported by adaptive MoE-driven dynamics and inter-dimensional cross attention, provides a stable and expressive mapping that directly operates in continuous function spaces, which the ablation study in Table~\ref{tab:1} clearly demonstrates. Table~\ref{tab:2} further evaluates model robustness across a variety of stress-test scenarios—including feature heterogeneity, noise levels, output structures, and input dimensionalities—demonstrating that \textsc{FAME} consistently maintains superior predictive performance under challenging conditions. Table~\ref{tab:3} reports results on real-world datasets, confirming the practical applicability of our model. The performance advantage over strong baselines is reduced compared to the synthetic setting.  Figure~\ref{fig:function}(a) shows that the \textsc{Hawaii Ocean} dataset occupies a smooth, low-variability subspace dictated by ocean stratification, a structure that fixed-basis and kernel methods can already approximate competently. Figure~\ref{fig:function}(b) contrasts this with synthetic trajectories drawn from independent Gaussian processes, whose broad, irregular function space contains pronounced local fluctuations. Under this more demanding regime, the continuous-attention encoder in \textsc{FAME} secures a substantially larger accuracy gain, underscoring its flexibility across disparate function-space complexities and corroborating the universal-approximation and sampling-invariance guarantees established in Section~\ref{sec:cross-att}. 

As shown in Figure~\ref{num}, the model’s performance varies systematically with key parameters, illustrating clear relationships between predictive behavior and data characteristics. Specifically, in Case~1, the observed test error progressively decreases with increasing sample size (Fig.~\ref{num}(a)), in agreement with our theoretical generalisation result. Case~2 illustrates a clear improvement in model performance as the number of sampling points grows, with \textsc{FAME} achieving superior accuracy even at a sparse resolution of 10 sampling points (Fig.~\ref{num}(b)). Case~3 reveals that while baseline methods experience substantial performance degradation under mixed input resolutions, \textsc{FAME} maintains stable and superior  performance. This result highlights our model’s sampling invariance capability, endowed by the Young integral formulation.
 As the input becomes noisier or more high-dimensional (Cases~4 and 7), all models exhibit increased error (Fig.~\ref{num}(c),(d)), but \textsc{FAME} remains within a practically useful error range. Taken together, these parameter sweeps confirm that \textsc{FAME}'s empirical behaviour matches its theoretical guarantees, maintaining robust accuracy across diverse data regimes. To assess how the number of experts \(K\) affects performance, we plot the test MSE together with the normalized routing entropy \(\tilde H = H/\log K\) in Figure~\ref{moe}, where \(H\) denotes the Shannon entropy of the gating distribution and the normalization by \(\log K\) yields a \([0,1]\)-scaled quantity comparable across different \(K\). The curves show that accuracy improves as \(K\) increases from 1 to 3 and plateaus around \(K=3\!\sim\!5\), while \(\tilde H\) remains moderate—indicating healthy, non-collapsed expert utilization. Beyond that, accuracy does not meaningfully improve and compute grows with \(K\). Consequently, \(K=3\) and \(K=5\) show stable accuracy and healthy expert utilization; since compute scales with \(K\), we adopt \(K=3\) as the default. More extensive visualizations and discussions are provided in Appendix~\ref{experiment_2}, which also details the experimental setup ( hyperparameters, and training schedules) and a comprehensive analysis of computational efficiency (time and memory complexity).

\begin{figure*}[t]
	\begin{center}
		\centerline{\includegraphics[width=\columnwidth]{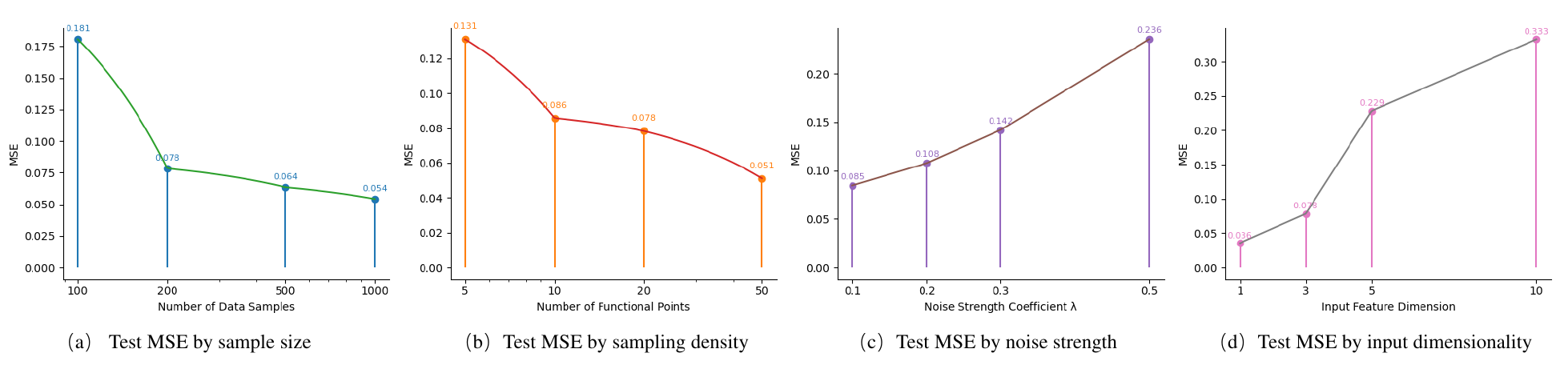}}
		\caption{Parameter-sensitivity curves for FAME.}
		\label{num}	
	\end{center}
\vskip -0.3in
\end{figure*}
\begin{figure*}[t]
	\begin{center}
		\centerline{\includegraphics[width=\columnwidth]{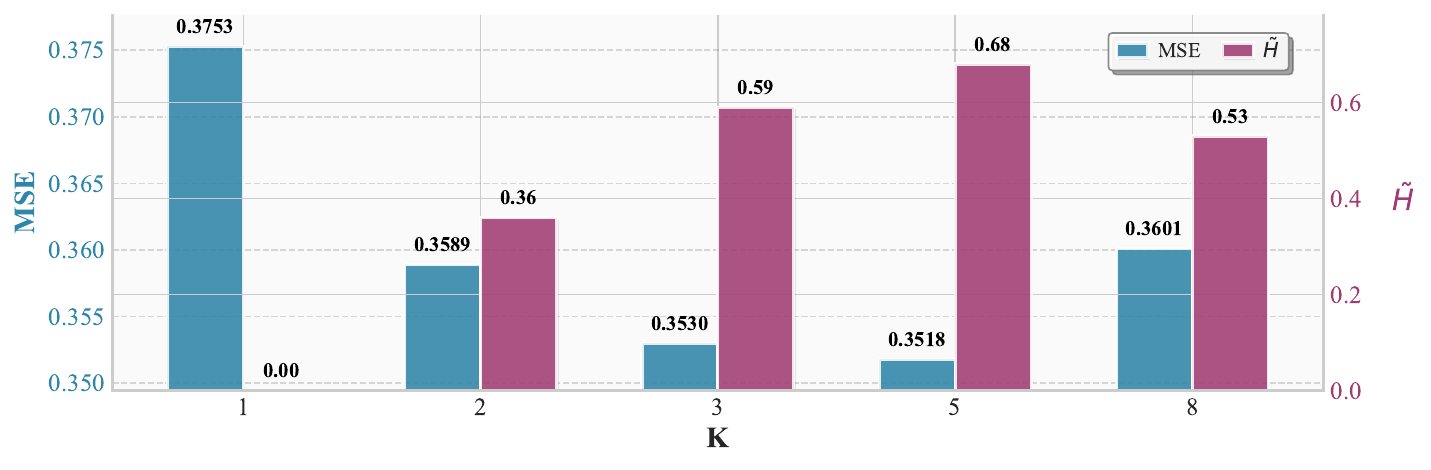}}
		\caption{Sensitivity to $K$—test MSE (left axis) and $\tilde H$ (right axis).}
		
		\label{moe}	
	\end{center}
	\vskip -0.3in
\end{figure*}

\section{Conclusion}
We have presented \textbf{FAME}, a fully end-to-end framework that learns FoFR mappings by coupling bidirectional neural CDEs with a continuous attention mechanism, enriching the resulting latent dynamics through a mixture-of-experts router, and fusing inter-functional information via multi-head cross attention before decoding with a NCDE.  This design yields a resolution-agnostic, Lipschitz-stable, and universally expressive operator that consistently outperforms classical bases, RKHS models, and recent deep networks on both synthetic and real benchmarks.  
A natural limitation of FAME is that when the underlying operator varies little across the function domain, simpler linear or low-rank models may offer a more attractive trade-off between statistical and computational complexity.  
In future work, we plan to extend the same functional attention principle beyond FoFR to tasks such as functional classification and function-on-scalar prediction, where the continuity-aware architecture of FAME is expected to provide similar advantages.

\bibliographystyle{Ref}  %plainnat,abbrvnat,unsrtnat
\small
\bibliography{Reference}

\begin{thebibliography}{50}
\providecommand{\natexlab}[1]{#1}
\providecommand{\url}[1]{\texttt{#1}}
\expandafter\ifx\csname urlstyle\endcsname\relax
  \providecommand{\doi}[1]{doi: #1}\else
  \providecommand{\doi}{doi: \begingroup \urlstyle{rm}\Url}\fi

\bibitem[Ardizzone et~al.(2018)Ardizzone, Kruse, Wirkert, Rahner, Pellegrini,
  Klessen, Maier-Hein, Rother, and K{\"o}the]{ardizzone2018analyzing}
Ardizzone , L., Kruse , J., Wirkert , S., Rahner , D., Pellegrini , E.~W.,
  Klessen , R.~S., Maier-Hein , L., Rother , C., \& K{\"o}the , U. (2018)
\newblock Analyzing inverse problems with invertible neural networks.
\newblock \emph{arXiv preprint arXiv:1808.04730}

\bibitem[Bahdanau et~al.(2014)Bahdanau, Cho, and Bengio]{bahdanau2014neural}
Bahdanau , D., Cho , K., \& Bengio , Y. (2014)
\newblock Neural machine translation by jointly learning to align and
  translate.
\newblock \emph{arXiv preprint arXiv:1409.0473}

\bibitem[Bellot and Van Der~Schaar(2021)]{bellot2021policy}
Bellot , A. \& Van Der~Schaar , M. (2021)
\newblock Policy analysis using synthetic controls in continuous-time. In
\newblock \emph{International Conference on Machine Learning}
\newblock pages 759--768. PMLR.

\bibitem[Bouche et~al.(2021)Bouche, Clausel, Roueff, and d’Alch{\'e}
  Buc]{bouche2021nonlinear}
Bouche , D., Clausel , M., Roueff , F., \& Buc , F. (2021)
\newblock Nonlinear functional output regression: A dictionary approach. In
\newblock \emph{International Conference on Artificial Intelligence and
  Statistics}
\newblock pages 235--243. PMLR.

\bibitem[Crambes et~al.(2009)Crambes, Kneip, and Sarda]{crambes2009smoothing}
Crambes , C., Kneip , A., \& Sarda , P. (2009)
\newblock Smoothing splines estimators for functional linear regression

\bibitem[Fan et~al.(2015)Fan, James, and Radchenko]{fan2015functional}
Fan , Y., James , G.~M., \& Radchenko , P. (2015)
\newblock Functional additive regression

\bibitem[Febrero-Bande et~al.(2019)Febrero-Bande, Gonz{\'a}lez-Manteiga, and
  Oviedo De La~Fuente]{febrero2019variable}
Febrero-Bande , M., Gonz{\'a}lez-Manteiga , W., \& Oviedo De La~Fuente , M.
  (2019)
\newblock Variable selection in functional additive regression models.
\newblock \emph{Computational Statistics} {\bfseries 34}:\penalty0 469--487.

\bibitem[Hall and Horowitz(2007)]{hall2007methodology}
Hall , P. \& Horowitz , J.~L. (2007)
\newblock Methodology and convergence rates for functional linear regression

\bibitem[Hullait et~al.(2021)Hullait, Leslie, Pavlidis, and
  King]{hullait2021robust}
Hullait , H., Leslie , D.~S., Pavlidis , N.~G., \& King , S. (2021)
\newblock Robust function-on-function regression.
\newblock \emph{Technometrics} {\bfseries 63}\penalty0 (3):\penalty0 396--409.

\bibitem[Ionescu et~al.(2013)Ionescu, Papava, Olaru, and
  Sminchisescu]{ionescu2013human3}
Ionescu , C., Papava , D., Olaru , V., \& Sminchisescu , C. (2013)
\newblock Human3. 6m: Large scale datasets and predictive methods for 3d human
  sensing in natural environments.
\newblock \emph{IEEE transactions on pattern analysis and machine intelligence}
  {\bfseries 36}\penalty0 (7):\penalty0 1325--1339.

\bibitem[Jank and Zhang(2011)]{jank2011automated}
Jank , W. \& Zhang , S. (2011)
\newblock An automated and data-driven bidding strategy for online auctions.
\newblock \emph{INFORMS Journal on computing} {\bfseries 23}\penalty0
  (2):\penalty0 238--253.

\bibitem[Jia and Benson(2019)]{jia2019neural}
Jia , J. \& Benson , A.~R. (2019)
\newblock Neural jump stochastic differential equations.
\newblock \emph{Advances in Neural Information Processing Systems} {\bfseries
  32}.

\bibitem[Jiang et~al.(2020)Jiang, Cheng, Yin, and Shen]{jiang2020functional}
Jiang , F., Cheng , Q., Yin , G., \& Shen , H. (2020)
\newblock Functional censored quantile regression.
\newblock \emph{Journal of the American Statistical Association} {\bfseries
  115}\penalty0 (530):\penalty0 931--944.

\bibitem[Kadri et~al.(2016)Kadri, Duflos, Preux, Canu, Rakotomamonjy, and
  Audiffren]{kadri2016operator}
Kadri , H., Duflos , E., Preux , P., Canu , S., Rakotomamonjy , A., \&
  Audiffren , J. (2016)
\newblock Operator-valued kernels for learning from functional response data.
\newblock \emph{Journal of Machine Learning Research} {\bfseries 17}\penalty0
  (20):\penalty0 1--54.

\bibitem[Kidger et~al.(2020)Kidger, Morrill, Foster, and
  Lyons]{kidger2020neural}
Kidger , P., Morrill , J., Foster , J., \& Lyons , T. (2020)
\newblock Neural controlled differential equations for irregular time series.
\newblock \emph{Advances in Neural Information Processing Systems} {\bfseries
  33}:\penalty0 6696--6707.

\bibitem[Lambert et~al.(2022)Lambert, Bouche, Szabo, and d’Alché
  Buc]{lambert2022functional}
Lambert , A., Bouche , D., Szabo , Z., \& Buc , F. (2022)
\newblock Functional output regression with infimal convolution: Exploring the
  huber and {\(\varepsilon\)}-insensitive losses. In
\newblock \emph{International Conference on Machine Learning}
\newblock pages 11844--11867. PMLR.

\bibitem[Lan et~al.(2023)Lan, Li, Li, Bai, Li, Tsung, Ketter, Zhao, and
  Zhang]{lan2023mm}
Lan , T., Li~, Z., Li~, Z., Bai , L., Li~, M., Tsung , F., Ketter , W., Zhao ,
  R., \& Zhang , C. (2023)
\newblock Mm-dag: Multi-task dag learning for multi-modal data-with application
  for traffic congestion analysis. In
\newblock \emph{Proceedings of the 29th ACM SIGKDD Conference on Knowledge
  Discovery and Data Mining}
\newblock pages 1188--1199.

\bibitem[Ledoux and Talagrand(1991)]{ledoux1991probability}
Ledoux , M. \& Talagrand , M. (1991)
\newblock \emph{Probability in Banach Spaces: Isoperimetry and Processes}.
\newblock \emph{Probability in Banach Spaces: Isoperimetry and Processes}:
  Springer.

\bibitem[Lian(2007)]{lian2007nonlinear}
Lian , H. (2007)
\newblock Nonlinear functional models for functional responses in reproducing
  kernel hilbert spaces.
\newblock \emph{Canadian Journal of Statistics} {\bfseries 35}\penalty0
  (4):\penalty0 597--606.

\bibitem[Liu et~al.(2023{\natexlab{a}})Liu, You, and Cao]{liu2023dynamic}
Liu , H., You , J., \& Cao , J. (2023.
\newblock {\natexlab{a}})
\newblock A dynamic interaction semiparametric function-on-scalar model.
\newblock \emph{Journal of the American Statistical Association} {\bfseries
  118}\penalty0 (541):\penalty0 360--373.

\bibitem[Liu et~al.(2023{\natexlab{b}})Liu, Hu, Zhang, Wu, Wang, Ma, and
  Long]{liu2023itransformer}
Liu , Y., Hu~, T., Zhang , H., Wu~, H., Wang , S., Ma~, L., \& Long , M. (2023.
\newblock {\natexlab{b}})
\newblock itransformer: Inverted transformers are effective for time series
  forecasting.
\newblock \emph{arXiv preprint arXiv:2310.06625}

\bibitem[Luo and Qi(2017)]{luo2017function}
Luo , R. \& Qi~, X. (2017)
\newblock Function-on-function linear regression by signal compression.
\newblock \emph{Journal of the American Statistical Association} {\bfseries
  112}\penalty0 (518):\penalty0 690--705.

\bibitem[Luo et~al.(2016)Luo, Qi, and Wang]{luo2016functional}
Luo , R., Qi~, X., \& Wang , Y. (2016)
\newblock Functional wavelet regression for linear function-on-function models

\bibitem[M{\"u}ller and Yao(2008)]{muller2008functional}
M{\"u}ller , H.-G. \& Yao , F. (2008)
\newblock Functional additive models.
\newblock \emph{Journal of the American Statistical Association} {\bfseries
  103}\penalty0 (484):\penalty0 1534--1544.

\bibitem[Mutis et~al.(2025)Mutis, Beyaztas, Karaman, and
  Lin~Shang]{mutis2025function}
Mutis , M., Beyaztas , U., Karaman , F., \& Lin~Shang , H. (2025)
\newblock On function-on-function linear quantile regression.
\newblock \emph{Journal of Applied Statistics} {\bfseries 52}\penalty0
  (4):\penalty0 814--840.

\bibitem[Oh et~al.(2024)Oh, Lim, and Kim]{oh2024stable}
Oh~, Y., Lim , D., \& Kim , S. (2024)
\newblock Stable neural stochastic differential equations in analyzing
  irregular time series data.
\newblock \emph{arXiv preprint arXiv:2402.14989}

\bibitem[Park et~al.(2022)Park, Ahn, and Jeon]{park2022sparse}
Park , J., Ahn , J., \& Jeon , Y. (2022)
\newblock Sparse functional linear discriminant analysis.
\newblock \emph{Biometrika} {\bfseries 109}\penalty0 (1):\penalty0 209--226.

\bibitem[Qi and Luo(2019)]{qi2019nonlinear}
Qi~, X. \& Luo , R. (2019)
\newblock Nonlinear function-on-function additive model with multiple predictor
  curves.
\newblock \emph{Statistica Sinica} {\bfseries 29}\penalty0 (2):\penalty0
  719--739.

\bibitem[Ramsay and Silverman(2007)]{ramsay2007applied}
Ramsay , J. \& Silverman , B. (2007)
\newblock \emph{Applied Functional Data Analysis: Methods and Case Studies}.
\newblock \emph{Applied Functional Data Analysis: Methods and Case Studies}:
  Springer.

\bibitem[Rossi and Conan-Guez(2006)]{rossi2006theoretical}
Rossi , F. \& Conan-Guez , B. (2006)
\newblock Theoretical properties of projection based multilayer perceptrons
  with functional inputs.
\newblock \emph{Neural Processing Letters} {\bfseries 23}\penalty0
  (1):\penalty0 55--70.

\bibitem[Rubanova et~al.(2019)Rubanova, Chen, and Duvenaud]{rubanova2019latent}
Rubanova , Y., Chen , R.~T., \& Duvenaud , D.~K. (2019)
\newblock Latent ordinary differential equations for irregularly-sampled time
  series.
\newblock \emph{Advances in neural information processing systems} {\bfseries
  32}.

\bibitem[Shi and Choi(2011)]{shi2011gaussian}
Shi , J.~Q. \& Choi , T. (2011)
\newblock \emph{Gaussian process regression analysis for functional data}.
\newblock \emph{Gaussian process regression analysis for functional data}: CRC
  press.

\bibitem[Shi et~al.(2024)Shi, Fan, Song, Zhou, and Suykens]{shi2024nonlinear}
Shi , Z., Fan , J., Song , L., Zhou , D.-X., \& Suykens , J.~A. (2024)
\newblock Nonlinear functional regression by functional deep neural network
  with kernel embedding.
\newblock \emph{arXiv preprint arXiv:2401.02890}

\bibitem[Stone(1948)]{stone1948weierstrass}
Stone , M.~H. (1948)
\newblock The generalized weierstrass approximation theorem.
\newblock \emph{Mathematics Magazine} {\bfseries 21}\penalty0 (5):\penalty0
  237--254.

\bibitem[Sun et~al.(2018)Sun, Du, Wang, and Ma]{sun2018optimal}
Sun , X., Du~, P., Wang , X., \& Ma~, P. (2018)
\newblock Optimal penalized function-on-function regression under a reproducing
  kernel hilbert space framework.
\newblock \emph{Journal of the American Statistical Association} {\bfseries
  113}\penalty0 (524):\penalty0 1601--1611.

\bibitem[Thind et~al.(2023)Thind, Multani, and Cao]{thind2023deep}
Thind , B., Multani , K., \& Cao , J. (2023)
\newblock Deep learning with functional inputs.
\newblock \emph{Journal of Computational and Graphical Statistics} {\bfseries
  32}\penalty0 (1):\penalty0 171--180.

\bibitem[Vaswani et~al.(2017)Vaswani, Shazeer, Parmar, Uszkoreit, Jones, Gomez,
  Kaiser, and Polosukhin]{vaswani2017attention}
Vaswani , A., Shazeer , N., Parmar , N., Uszkoreit , J., Jones , L., Gomez ,
  A.~N., Kaiser , {\L}., \& Polosukhin , I. (2017)
\newblock Attention is all you need.
\newblock \emph{Advances in neural information processing systems} {\bfseries
  30}.

\bibitem[Wang et~al.(2016)Wang, Chiou, and M{\"u}ller]{wang2016functional}
Wang , J.-L., Chiou , J.-M., \& M{\"u}ller , H.-G. (2016)
\newblock Functional data analysis.
\newblock \emph{Annual Review of Statistics and its application} {\bfseries
  3}\penalty0 (1):\penalty0 257--295.

\bibitem[Wang et~al.(2019)Wang, Zheng, Farahat, Serita, Saeki, and
  Gupta]{wang2019multilayer}
Wang , Q., Zheng , S., Farahat , A., Serita , S., Saeki , T., \& Gupta , C.
  (2019)
\newblock Multilayer perceptron for sparse functional data. In
\newblock \emph{2019 International joint conference on neural networks (IJCNN)}
\newblock pages 1--10. IEEE.

\bibitem[Wang et~al.(2020)Wang, Wang, Gupta, Rao, and Khorasgani]{wang2020non}
Wang , Q., Wang , H., Gupta , C., Rao , A.~R., \& Khorasgani , H. (2020)
\newblock A non-linear function-on-function model for regression with time
  series data. In
\newblock \emph{2020 IEEE International Conference on Big Data (Big Data)}
\newblock pages 232--239. IEEE.

\bibitem[Wang et~al.(2023)Wang, Cao, Shang, and Initiative]{wang2023deep}
Wang , S., Cao , G., Shang , Z., \& Initiative , A. D.~N. (2023)
\newblock Deep neural network classifier for multidimensional functional data.
\newblock \emph{Scandinavian Journal of Statistics} {\bfseries 50}\penalty0
  (4):\penalty0 1667--1686.

\bibitem[Wen et~al.(2023)Wen, Wang, and Zhang]{wen2023}
Wen , C., Wang , X., \& Zhang , A. (2023)
\newblock $\ell$0 trend filtering.
\newblock \emph{INFORMS Journal on Computing} {\bfseries 35}\penalty0
  (6):\penalty0 1491--1510.

\bibitem[Xu et~al.(2015)Xu, Ba, Kiros, Cho, Courville, Salakhudinov, Zemel, and
  Bengio]{xu2015show}
Xu~, K., Ba~, J., Kiros , R., Cho , K., Courville , A., Salakhudinov , R.,
  Zemel , R., \& Bengio , Y. (2015)
\newblock Show, attend and tell: Neural image caption generation with visual
  attention. In
\newblock \emph{International conference on machine learning}
\newblock pages 2048--2057. PMLR.

\bibitem[Yao et~al.(2005)Yao, M{\"u}ller, and Wang]{yao2005functional}
Yao , F., M{\"u}ller , H.-G., \& Wang , J.-L. (2005)
\newblock Functional linear regression analysis for longitudinal data

\bibitem[Yao et~al.(2021)Yao, Mueller, and Wang]{yao2021deep}
Yao , J., Mueller , J., \& Wang , J.-L. (2021)
\newblock Deep learning for functional data analysis with adaptive basis
  layers. In
\newblock \emph{International Conference on Machine Learning}
\newblock pages 11898--11908. PMLR.

\bibitem[Young(1936)]{young1936inequality}
Young , L.~C. (1936)
\newblock An inequality of the h{\"o}lder type, connected with stieltjes
  integration

\bibitem[Zhang and Yan(2023)]{zhang2023crossformer}
Zhang , Y. \& Yan , J. (2023)
\newblock Crossformer: Transformer utilizing cross-dimension dependency for
  multivariate time series forecasting. In
\newblock \emph{The eleventh international conference on learning
  representations}

\bibitem[Zhang et~al.(2022)Zhang, Wang, Kong, and Zhu]{zhang2022high}
Zhang , Z., Wang , X., Kong , L., \& Zhu , H. (2022)
\newblock High-dimensional spatial quantile function-on-scalar regression.
\newblock \emph{Journal of the American Statistical Association} {\bfseries
  117}\penalty0 (539):\penalty0 1563--1578.

\bibitem[Zhao et~al.(2012)Zhao, Ogden, and Reiss]{zhao2012wavelet}
Zhao , Y., Ogden , R.~T., \& Reiss , P.~T. (2012)
\newblock Wavelet-based lasso in functional linear regression.
\newblock \emph{Journal of computational and graphical statistics} {\bfseries
  21}\penalty0 (3):\penalty0 600--617.

\bibitem[Zhou et~al.(2021)Zhou, Zhang, Peng, Zhang, Li, Xiong, and
  Zhang]{haoyietal-informer-2021}
Zhou , H., Zhang , S., Peng , J., Zhang , S., Li~, J., Xiong , H., \& Zhang ,
  W. (2021)
\newblock Informer: Beyond efficient transformer for long sequence time-series
  forecasting. In
\newblock \emph{The Thirty-Fifth {AAAI} Conference on Artificial Intelligence,
  {AAAI} 2021, Virtual Conference}
\newblock \emph{35}, pp. \penalty0 11106--11115. {AAAI} Press.

\end{thebibliography}
\normalsize

	%%%%%%%%%%%%%%%%%%%%%%%%%%%%%%%%%%%%%%%%%%%%%%%%%%%%%%%%%%%%

\appendix
\newpage
%==============================================================
%==============================================================
\section{Theoretical Guarantees of FAME}\label{app:theory}
%==============================================================
This appendix presents complete proofs for all theoretical claims made in the main text.  
We begin by deriving Lipschitz constants for each architectural block and compose them to
obtain an end-to-end stability bound.  We then establish sampling invariance and universal
approximation, and conclude with a quantitative generalisation result based on Rademacher
complexity.

\textbf{Notation.}
Each input function \(X^{(j)}\) is defined on its own horizon \([0,T_j]\); we write
\[
\mathcal X
\;=\;
\prod_{j=1}^{d} C\!\bigl([0,T_j],\mathbb R\bigr),
\qquad
\|X\|_{1\text{-var}}
\;=\;
\max_{j}\bigl\|X^{(j)}\bigr\|_{1\text{-var}} .
\]
Each output function \(Y^{(\zeta)}\) lives on \([0,S_\zeta]\); correspondingly,
\[
\mathcal Y
\;=\;
\prod_{\zeta=1}^{m} C\!\bigl([0,S_\zeta],\mathbb R\bigr),
\qquad
\|Y\|_\infty \;=\; \max_{\zeta}\|Y^{(\zeta)}\|_\infty .
\]
For any (vector-valued) path \(X\), \(\|X\|_{1\text{-var}}\) denotes the usual
total variation and is taken component-wise as above.
All integrals are understood in the Young sense \citep{young1936inequality},
so they depend only on the underlying continuous functions, not on the choice of discretisation.

%--------------------------------------------------------------
\subsection{Module-wise Lipschitz constants}\label{app:lip}

\paragraph{(i) Continuous attention.}
Under Assumption~\ref{ass:bidir-lip} the bidirectional NCDE (forward
\emph{and} backward flows) is well posed for every input function.
Theorem~\ref{thm:bidir-exist} in the main text shows that the mapping
\(X\!\mapsto\!\hat Z\) satisfies
\[
\|\hat Z-\tilde{\hat Z}\|_\infty
\;\le\;
L_{\text{single}}\,
\|X-\tilde X\|_{1\text{-var}},
\qquad
L_{\text{single}}
=\frac{\|W_V\|}{\gamma}\,
e^{L_j(T_j-t_0)}\,L_j(T_j-t_0) .
\]

\paragraph{(ii) Cross attention.}
Let \(\sigma(a)=\mathrm{softmax}\!\bigl(a/\sqrt{d_c}\bigr)\).

\begin{lemma}[Soft-max contraction]\label{lem:softmax}
	The map \(\sigma:\mathbb R^{d_c}\!\to\!\Delta^{d_c-1}\) is
	\(1/\sqrt{d_c}\)-Lipschitz from \((\ell_\infty,\|\cdot\|_\infty)\) to
	\((\ell_1,\|\cdot\|_1)\).
\end{lemma}

\begin{proof}
	The Jacobian of \(\sigma\) is
	\(\mathrm{diag}(\sigma)-\sigma\sigma^{\!\top}\); its operator norm
	\(\ell_\infty\!\to\!\ell_1\) equals \(1/\sqrt{d_c}\).
	Integrating this bound along the line segment joining any two inputs yields the claim.
\end{proof}

If the projection matrices satisfy
\(\|W^{(p)}_{Q}\|_{\mathrm{op}},\|W^{(p)}_{K}\|_{\mathrm{op}},
\|W^{(p)}_{V}\|_{\mathrm{op}},\|W_{O}\|_{\mathrm{op}}\le M_{mat}\) for every head~\(p\),
Lemma~\ref{lem:softmax} implies, path-wise in \(1\)-variation,
\[
\|H-\tilde H\|_{1\text{-var}}
\;\le\;
L_{\text{cross}}\,
\|\hat Z-\tilde{\hat Z}\|_{1\text{-var}},
\qquad
L_{\text{cross}}=\frac{M_{mat}^{\,3}M}{\sqrt{d_c}} .
\]

\paragraph{(iii) CDE decoder.}
If \(f_\psi\) is globally \(L_\psi\)-Lipschitz and bounded, the
Young–Löwner framework together with Grönwall’s inequality yields, for
every output function \(\zeta\),
\[
\|\hat Y^{(\zeta)}-\tilde Y^{(\zeta)}\|_\infty
\;\le\;
L_{\text{dec}}\,
\|H-\tilde H\|_{1\text{-var}},
\qquad
L_{\text{dec}} = e^{L_\psi S_\zeta}\,L_\psi S_\zeta .
\]

\smallskip
Writing \(S_{\max}=\max_\zeta S_\zeta\), the end-to-end Lipschitz radius is
\[
L_{*}=L_{\text{dec}}\,L_{\text{cross}}\,L_{\text{single}},
\qquad
L_{\text{dec}}=e^{L_\psi S_{\max}}L_\psi S_{\max}.
\]

%--------------------------------------------------------------
\subsection{End-to-end stability and sampling invariance}\label{app:stab}

\begin{theorem}[Global Lipschitz bound]\label{thm:globalLip}
	For the composite operator
	\(\mathcal T_\Theta=\mathcal D_\psi\circ\mathcal E_\theta\),
	\[
	\|\mathcal T_\Theta(X)-\mathcal T_\Theta(\tilde X)\|_\infty
	\;\le\;
	L_{*}\,
	\|X-\tilde X\|_{1\text{-var}}.
	\]
\end{theorem}

\begin{proof}
	Chain the inequalities of Section~\ref{app:lip} along
	\(X\!\to\!\hat Z\!\to\!H\!\to\!\hat Y\).
\end{proof}

\begin{proposition}[Sampling invariance]\label{prop:sample}
	If two observation grids encode the same underlying functions, the
	operator \(\mathcal T_\Theta\) returns identical outputs.
\end{proposition}

\begin{proof}
	Young integrals depend only on the driving functions, not on the chosen
	partitions \citep{young1936inequality}.
\end{proof}

%--------------------------------------------------------------
%--------------------------------------------------------------
%--------------------------------------------------------------
\subsection{Universal approximation}\label{app:uat}

\begin{lemma}[Density of separable kernels]\label{lem:kernel}
	Define
	\begin{equation}\label{eq:defS}
		\mathcal S
		\;=\;
		\Bigl\{
		k(t,\tau)=\textstyle\sum_{i=1}^{r}u_i(t)\,v_i(\tau)
		\;\Bigm|\;
		r\in\mathbb N,\;
		u_i,v_i\in C\!\bigl([t_0,T_j]\bigr)
		\Bigr\}.
	\end{equation}
	The uniform closure of\/~$\mathcal S$ equals the full kernel space:
	\[
	\overline{\mathcal S}=C\!\bigl([t_0,T_j]^2\bigr).
	\]
\end{lemma}

\begin{proof}
	\textit{(i) Sub-algebra.}  
	$\mathcal S$ is closed under point-wise addition, scalar multiplication,
	and multiplication, so it is a sub-algebra of
	$C\!\bigl([t_0,T_j]^2\bigr)$.
	
	\smallskip
	\noindent\textit{(ii) Constants.}  
	Choosing $u\equiv 1$ and $v\equiv c$ gives the constant kernel
	$k(t,\tau)\equiv c\in\mathcal S$.
	
	\smallskip
	\noindent\textit{(iii) Point separation.}  
	For two distinct points $(t_1,\tau_1)\neq(t_2,\tau_2)$:
	\begin{itemize}
		\item If $t_1\neq t_2$, pick $u\in C([t_0,T_j])$ with
		$u(t_1)\neq u(t_2)$ (by Urysohn’s lemma) and set $v\equiv 1$.
		\item If $t_1=t_2$ (hence $\tau_1\neq\tau_2$), pick
		$v\in C([t_0,T_j])$ with
		$v(\tau_1)\neq v(\tau_2)$ and set $u\equiv 1$.
	\end{itemize}
	In both cases the resulting kernel belongs to~$\mathcal S$ and takes
	different values at the two points, so $\mathcal S$ separates points of
	the compact Hausdorff space $[t_0,T_j]^2$.
	
	\smallskip
	\noindent\textit{(iv) Stone–Weierstrass.}  
	Because $\mathcal S$ is a sub-algebra that contains the constants and
	separates points, the real Stone–Weierstrass
	theorem~\citep{stone1948weierstrass} yields
	$\overline{\mathcal S}=C\!\bigl([t_0,T_j]^2\bigr)$.
\end{proof}

\begin{theorem}[Density of the FAME hypothesis class]\label{thm:uat}
	For any bounded continuous operator
	$\mathcal F:\mathcal X\to\mathcal Y$ and any $\varepsilon>0$
	there exists a parameter set\/~$\Theta$ such that
	\[
	\|\mathcal T_\Theta-\mathcal F\|_\infty<\varepsilon.
	\]
\end{theorem}

\begin{proof}
	Neural CDEs are universal approximators on path space
	\citep{kidger2020neural}.  
	Lemma~\ref{lem:kernel} shows that cross-attention can approximate any
	continuous kernel, and the decoder NCDE is a universal curve generator.
	Universality is preserved under composition, so the class
	$\{\mathcal T_\Theta\}$ is dense in
	$C\!\bigl(\mathcal X,\mathcal Y\bigr)$.
\end{proof}

%--------------------------------------------------------------
\subsection{Generalisation bound}\label{app:gen}

\begin{theorem}[Rademacher complexity]\label{thm:rad}
	For \(N\) i.i.d.\ samples, the hypothesis class with radius \(L_{*}\)
	satisfies \(\mathfrak R_N \le c\,L_{*}/\sqrt N\) for a universal constant
	\(c\).
\end{theorem}

\begin{proof}
	Apply the Ledoux–Talagrand contraction principle
	\citep{ledoux1991probability} to Theorem~\ref{thm:globalLip}.
\end{proof}

\section{Experimental details}
\label{experiment_2}

\subsection{Real-world Dataset}
\label{sec:real-world}

\paragraph{Hawaii ocean dataset.}
The Hawaii ocean dataset, part of the Hawaii Ocean Time-Series program, includes measurements of hydrographic, chemical, and biological characteristics at a station north of Oahu, Hawaii, from January 1, 1999, to December 31, 2018. The dataset consists of five functional variables: Temperature, Oxygen concentration, Potential density, Salinity, and Chloropigment, measured every 2 meters between 0 and 200 meters below the sea surface. After data preprocessing, we retained 265 samples, with irregular sampling (\(\Lambda_i=\Gamma_i = 20\)) as the number of observation points for each function sample. For our analysis, Temperature and Salinity alternately serve as the target variable, with the remaining variables used as input predictors.

\paragraph{Human3.6M dataset.}
The Human3.6M dataset \citep{ionescu2013human3} consists of 11 subjects performing 15 different actions, with 3D joint coordinates provided for 32 body parts. Each action is captured with high precision across different subjects, making the dataset suitable for studying human motion. In our approach, we enhance the dataset by performing data augmentation through random sampling of time series data from different subjects for the same action. Specifically, for each action, we sample 30 points, and each subject contributes 30 samples, ensuring a diverse set of data for training.

We have designed three specific regression tasks:
\begin{itemize}
	\item \textbf{Walking}: For the \emph{Walking} action, the input consists of the X and Y coordinates of the \emph{right knee}, and the output is the corresponding Z coordinate. This task includes 180 samples from multiple subjects.
	\item \textbf{Eating 1}: For the \emph{Eating 1} action, the input includes the XYZ coordinates of the \emph{right forearm} and the XY coordinates of the \emph{right wrist}, with the output being the Z coordinate of the \emph{right wrist}. This task also contains 180 samples.
	\item \textbf{Sitting down}: For the \emph{Sitting down} and \emph{Sitting down1} actions, the input consists of the XYZ coordinates of \emph{Spine1} and \emph{Spine2}, and the output is the XYZ coordinates of \emph{Spine3}. This task includes 390 samples.
\end{itemize}

\paragraph{ETT-small dataset.}
The ETT-small dataset \citep{haoyietal-informer-2021} contains data from two electricity transformer stations, including variables such as oil temperature and power load. Each sample is constructed from monthly data, comprising 30 data points, with each point representing a day's worth of measurements. For the experiments, we use a total of 48 samples, aiming to assess the performance of regular time sampling in this context.

\subsection{Experimental Setup}
All experiments were performed on a workstation equipped with an
AMD Ryzen 9 5950X CPU and an NVIDIA RTX 3090 GPU.  
Unless otherwise stated, every model is trained for $100$ epochs with the
Adam optimiser, an initial learning rate of $1\times10^{-3}$, and a
dropout rate of $0.2$.  
Each dataset is randomly split into $80\%$ training and $20\%$ test
instances; for the synthetic benchmark we repeat this split five times
and report the average performance.  For \textsc{FAME}, the neural controlled differential equation that
constructs the continuous attention and the decoder CDE share an
identical architecture.  
The input lifting network $\xi_\theta$ and the neural vector field
$f_\theta$ are both implemented as two–layer MLPs with hidden widths $32$
and $64$, respectively, followed by Tanh activations.  
The same MLP configuration is used for the MoE expert fields
and for the decoder vector field $f_\psi$, ensuring architectural
consistency throughout the model.

\subsection{Results and discussion}	

% ===================== Runtime Table =====================
\begin{table*}[t]
	\centering
	\caption{Runtime comparison (training / inference).}
	\label{tab:runtime}
	\small
	\begin{tabular}{lccc}
		\toprule
		\multicolumn{4}{c}{Case 1 (Train / Infer s)}\\
		\cmidrule(lr){1-4}
		Method & 100 & 200 & 500 \\
		\midrule
		Basis Expansion   & 0.01 / 0.01 & 0.01 / 0.01 & 0.01 / 0.01 \\
		FPCA              & 0.01 / 0.01 & 0.01 / 0.01 & 0.01 / 0.01 \\
		Kernel Method     & 0.01 / 0.01 & 0.01 / 0.01 & 0.01 / 0.01 \\
		Gaussian Process  & 0.03 / 0.03 & 0.03 / 0.03 & 0.03 / 0.03 \\
		FNN               & $0.40\times10$ / 0.36 & $0.62\times10$ / 0.37 & $0.97\times10$ / 0.36 \\
		FAME              & $48.79\times20$ (epoch) / 1.02 & $57.46\times20$ / 1.00 & $79.53\times20$ / 1.03 \\
		\midrule
		\multicolumn{4}{c}{Case 2 (Train / Infer s)}\\
		\cmidrule(lr){1-4}
		Method & 100 & 200 & 500 \\
		\midrule
		Basis Expansion   & 0.01 / 0.01 & 0.01 / 0.01 & 0.01 / 0.01 \\
		FPCA              & 0.01 / 0.01 & 0.01 / 0.01 & 0.01 / 0.01 \\
		Kernel Method     & 0.01 / 0.01 & 0.01 / 0.01 & 0.01 / 0.01 \\
		Gaussian Process  & 0.03 / 0.03 & 0.03 / 0.03 & 0.03 / 0.03 \\
		FNN               & $1.02\times10$ / 0.81 & $1.98\times10$ / 0.82 & $3.32\times10$ / 0.84 \\
		FAME              & $134.42\times20$ / 2.60 & $151.47\times20$ / 2.66 & $207.20\times20$ / 2.92 \\
		\bottomrule
	\end{tabular}
	\vspace{-2mm}
\end{table*}

% ===================== Memory Table =====================
\begin{table}[t]
	\centering
	\caption{Peak memory usage (MB).}
	\label{tab:memory}
	\small
	\begin{tabular}{lcccccc}
		\toprule
		Method & Basis Expansion & FPCA & Kernel Method & Gaussian Process & FNN & FAME \\
		\midrule
		Memory (MB) & 0.30 & 0.25 & 2.91 & 12.70 & 3.82 & 15.36 \\
		\bottomrule
	\end{tabular}
	\vspace{-2mm}
\end{table}

\label{experiment2}
While the main text already provides extensive quantitative evidence, additional results—spanning Tables~4--6 and Figures~5--7—together with analyses of stress tests and computational efficiency, merit inclusion for completeness.

\paragraph{\textbf{Supplementary Stress-Test Analysis}}
Across all supplementary simulations, \textsc{FAME} preserves a decisive lead over competing approaches.  
The heterogeneity configuration in Case 4 assigns distinct length-scales $\{0.2,0.3,0.5\}$ to the three input coordinates; ablating the mixture-of-experts router in this setting raises the test MSE from $0.0798$ to $0.0842$, demonstrating the value of specialist vector fields for channel-specific dynamics.  As the input dimensionality rises (case 7), the test
errors of all baselines increase markedly, highlighting how swiftly FoFR becomes more demanding in higher
dimensions.  Kernel methods, which performed strongly in the
low-dimensional settings of Cases 1–3, deteriorate the most.  Basis‐expansion pipelines and FPCA
regressors also degrade, since fixed bases struggle to model the richer
cross-coordinate interactions that emerge with additional functional
inputs.  Although \textsc{FAME} is likewise affected by the added
complexity, its relative performance drop is considerably smaller, and
its absolute accuracy remains within a practically useful range.  These
findings indicate that the continuous attention encoder and
mixture-of-experts vector fields endow \textsc{FAME} with a robustness
that scales more gracefully than existing alternatives when moving to
high-dimensional functional spaces.

\paragraph{\textbf{Loss Dynamics}}
Figure~\ref{fig:train-loss} traces the optimisation trajectory of \textsc{FAME}. The loss decreases monotonically and stabilises after roughly 20 epochs, demonstrating both rapid convergence and training stability for the continuous attention architecture.

\paragraph{Runtime and memory.}
Cases~1–2 jointly span the two principal axes of computational load—the number of input functions and the number of sampling points per function—so they are well suited for reporting cost. All timings in Table~\ref{tab:runtime} were obtained on the same ordinary desktop. In brief, basis/FPCA/kernel methods train fastest when grids are fixed and models are small, but their reliance on pre-specified bases or kernels can limit accuracy. By contrast, \textsc{FAME} is fully data-driven: training is costlier but remains lightweight in absolute terms, and inference latencies are comparable to classical baselines. Even in the most demanding configuration we considered (Case~2 with 500 samples), \textsc{FAME} trains in about $207.20$\,s per epoch; over 20 epochs this totals $\approx 4144$\,s ($\approx 69$\,min) on this CPU, while post-training inference remains at the seconds level (Table~\ref{tab:runtime}). Memory usage varies little across tasks and is modest overall (Table~\ref{tab:memory}); the peak footprint for \textsc{FAME} is only ${\sim}15$\,MB, and datasets are purely numeric and small on disk, making deployment straightforward.

\paragraph{\textbf{Regression Visualisation}}Figure~\ref{fig:sample} reports an analogous comparison for the synthetic Case 1 setting. Across the entire domain the predicted and true curves are almost indistinguishable, illustrating \textsc{FAME}’s ability to recover fine-grained function-to-function mappings even under irregular sampling. 
Figure~\ref{fig:regression} presents model outputs on the \textit{Sitting down} sequence from the Human3.6M benchmark—a representative real-world task. Basis-function baselines capture global trends, whereas conventional neural networks reproduce local variations; \textsc{FAME} accurately follows both scales and achieves the closest alignment with ground truth. 

\begin{figure}[t]
	\begin{center}
		\includegraphics[width=0.7\linewidth]{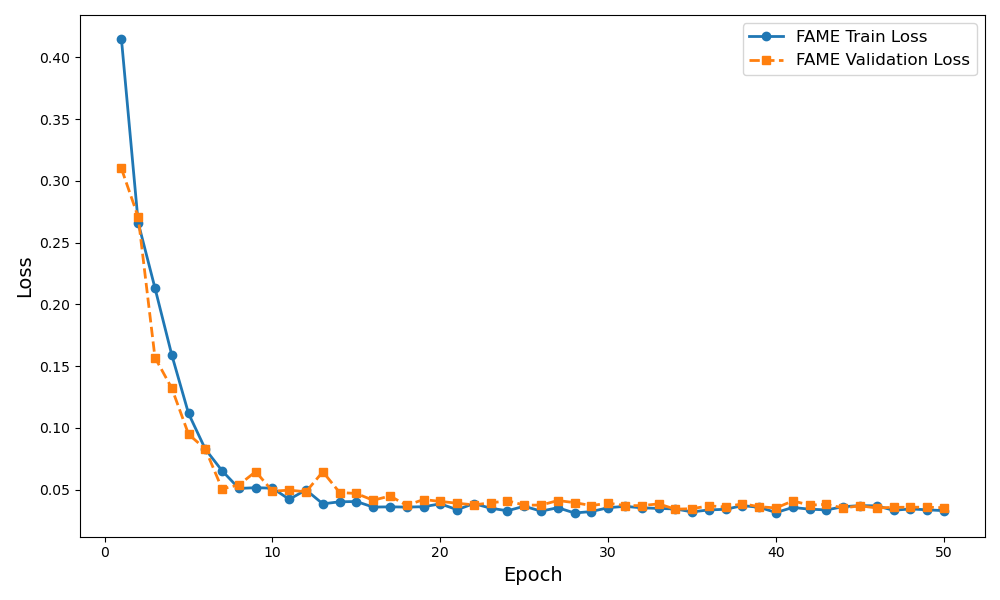}
		\caption{Training and validation loss over epochs. The monotonic convergence illustrates stable optimisation behaviour.}
		\label{fig:train-loss}
	\end{center}
\end{figure}

\begin{figure}[H]
	\begin{center}
		\includegraphics[width=\columnwidth]{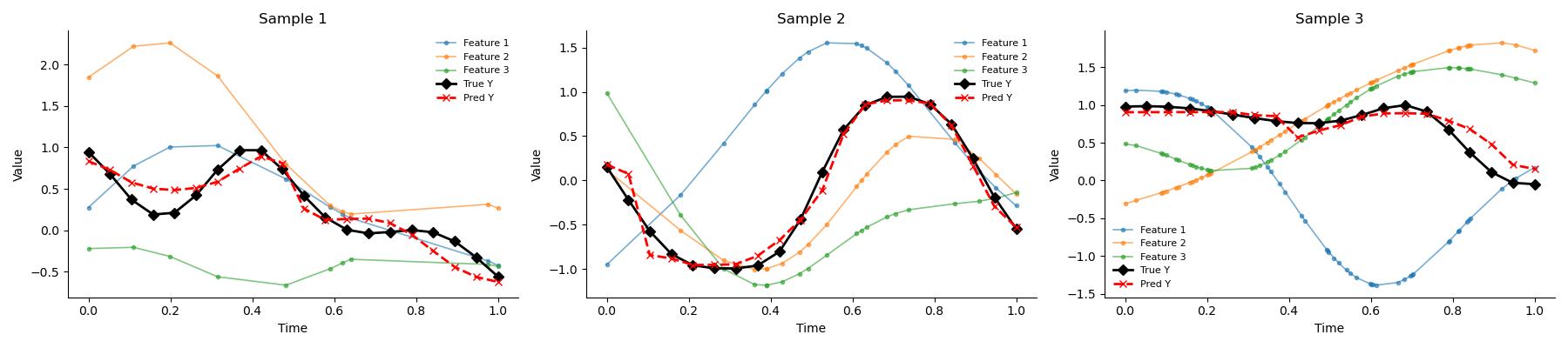}
		\caption{Prediction curves for \textit{Case 1}.
		}
		\label{fig:sample}
	\end{center}
\end{figure}

\newpage
\begin{figure}[H]
	\begin{center}
		\includegraphics[width=0.9\linewidth]{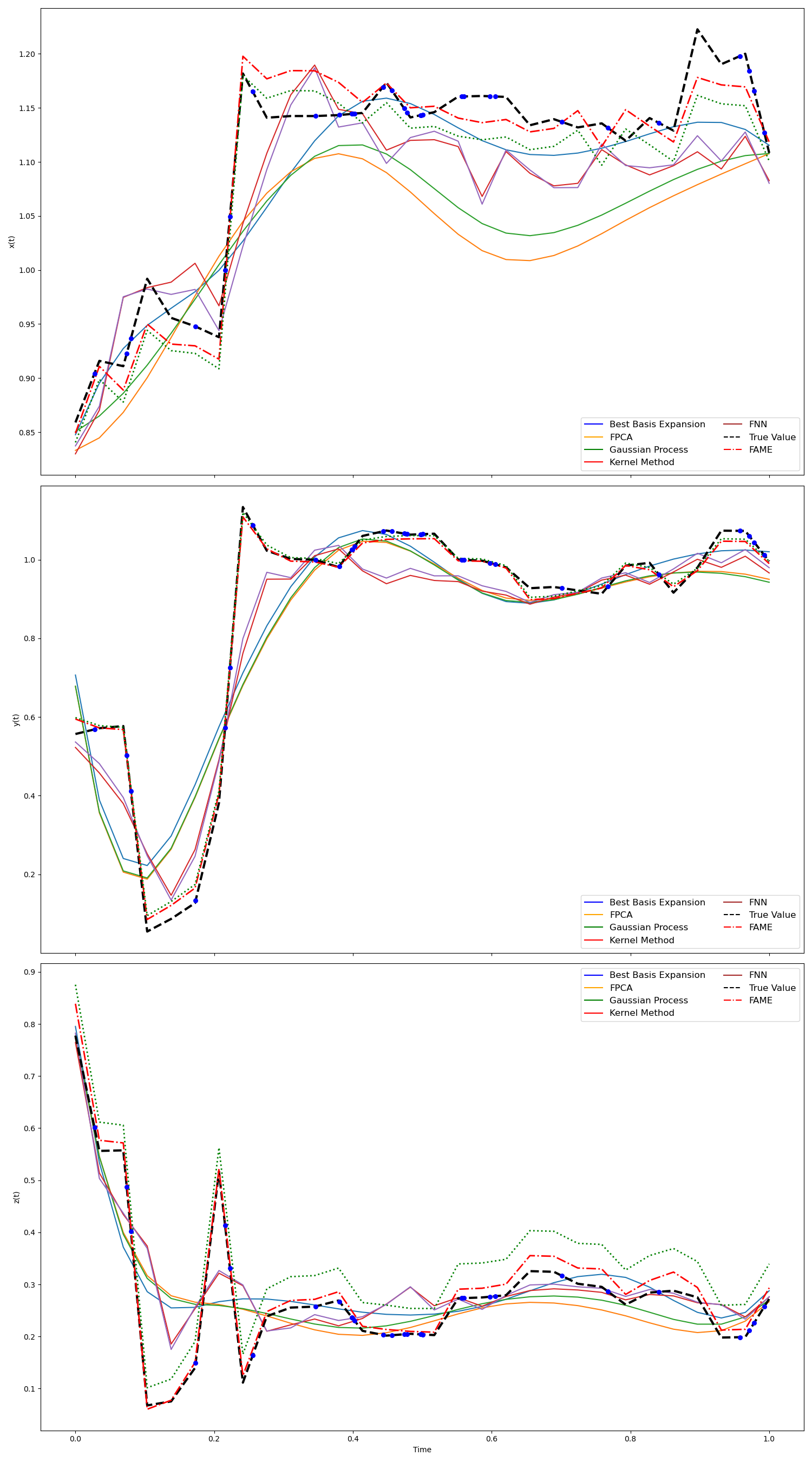}
		\caption{Prediction curves on the \textit{Sitting down} task (Human3.6M).}
		\label{fig:regression}
	\end{center}
\end{figure}

\begin{sidewaystable*}[!p]
	\setlength{\tabcolsep}{5pt}
	\caption{Average test MSE for different methods in regression, reported as mean $\pm$ standard deviation. The best MSE for each case is highlighted in bold.}
	\label{tab:1_d}
	\begin{scriptsize} % was \begin{small}
			\begin{center}
				\begin{tabular}{llccccccccc}
					\toprule
					\multirow{2}{*}{Model} & \multirow{2}{*}{} &
					\multicolumn{3}{c}{Case 1} & \multicolumn{3}{c}{Case 2} & \multicolumn{3}{c}{Case 3} \\
					\cmidrule(lr){3-5} \cmidrule(lr){6-8} \cmidrule(lr){9-11}
					& & 100 & 200 & 500 & 100 & 200 & 500 & 100 & 200 & 500 \\
					\midrule
					\multirow{4}{*}{B-spline} & Linear
					& 0.4720 $\pm$ 0.0378 & 0.3947 $\pm$ 0.0316 & 0.3123 $\pm$ 0.0250
					& 0.4135 $\pm$ 0.0331 & 0.3412 $\pm$ 0.0273 & 0.2822 $\pm$ 0.0226
					& 0.4958 $\pm$ 0.0397 & 0.4002 $\pm$ 0.0320 & 0.3260 $\pm$ 0.0261 \\
					& Ridge
					& 0.4264 $\pm$ 0.0341 & 0.3893 $\pm$ 0.0311 & 0.3117 $\pm$ 0.0249
					& 0.3960 $\pm$ 0.0317 & 0.3393 $\pm$ 0.0271 & 0.2817 $\pm$ 0.0225
					& 0.3869 $\pm$ 0.0310 & 0.3740 $\pm$ 0.0299 & 0.3222 $\pm$ 0.0258 \\
					& Lasso
					& 0.4098 $\pm$ 0.0328 & 0.3830 $\pm$ 0.0306 & 0.3052 $\pm$ 0.0244
					& 0.3856 $\pm$ 0.0308 & 0.3351 $\pm$ 0.0268 & 0.2766 $\pm$ 0.0221
					& 0.4132 $\pm$ 0.0331 & 0.3584 $\pm$ 0.0287 & 0.3188 $\pm$ 0.0255 \\
					& Elastic Net
					& 0.4510 $\pm$ 0.0361 & 0.3650 $\pm$ 0.0292 & 0.2874 $\pm$ 0.0230
					& 0.3725 $\pm$ 0.0298 & 0.3152 $\pm$ 0.0252 & 0.2583 $\pm$ 0.0207
					& 0.3986 $\pm$ 0.0319 & 0.3618 $\pm$ 0.0289 & 0.3190 $\pm$ 0.0255 \\
					\multirow{4}{*}{Fourier} & Linear
					& 0.5002 $\pm$ 0.0400 & 0.4092 $\pm$ 0.0327 & 0.3224 $\pm$ 0.0258
					& 0.4923 $\pm$ 0.0394 & 0.3636 $\pm$ 0.0291 & 0.2896 $\pm$ 0.0232
					& 0.4762 $\pm$ 0.0381 & 0.3921 $\pm$ 0.0314 & 0.3547 $\pm$ 0.0284 \\
					& Ridge
					& 0.4255 $\pm$ 0.0340 & 0.3780 $\pm$ 0.0302 & 0.3149 $\pm$ 0.0252
					& 0.3616 $\pm$ 0.0289 & 0.3343 $\pm$ 0.0267 & 0.2841 $\pm$ 0.0227
					& 0.3755 $\pm$ 0.0300 & 0.3568 $\pm$ 0.0285 & 0.3328 $\pm$ 0.0266 \\
					& Lasso
					& 0.3550 $\pm$ 0.0284 & 0.3493 $\pm$ 0.0279 & 0.3135 $\pm$ 0.0251
					& 0.3361 $\pm$ 0.0269 & 0.3280 $\pm$ 0.0262 & 0.3001 $\pm$ 0.0240
					& 0.3808 $\pm$ 0.0305 & 0.3560 $\pm$ 0.0285 & 0.3247 $\pm$ 0.0260 \\
					& Elastic Net
					& 0.3540 $\pm$ 0.0283 & 0.3325 $\pm$ 0.0266 & 0.2914 $\pm$ 0.0233
					& 0.3167 $\pm$ 0.0253 & 0.3070 $\pm$ 0.0246 & 0.2720 $\pm$ 0.0218
					& 0.3737 $\pm$ 0.0299 & 0.3496 $\pm$ 0.0280 & 0.3366 $\pm$ 0.0269 \\
					\cmidrule(lr){0-1}
					\multicolumn{2}{c}{FPCA}
					& 0.3717 $\pm$ 0.0297 & 0.3554 $\pm$ 0.0284 & 0.3200 $\pm$ 0.0256
					& 0.2890 $\pm$ 0.0231 & 0.2733 $\pm$ 0.0219 & 0.2624 $\pm$ 0.0210
					& 0.3812 $\pm$ 0.0305 & 0.3563 $\pm$ 0.0285 & 0.3295 $\pm$ 0.0264 \\
					\multicolumn{2}{c}{Kernel Method}
					& 0.2441 $\pm$ 0.0195 & 0.1728 $\pm$ 0.0138 & 0.1058 $\pm$ 0.0085
					& 0.1741 $\pm$ 0.0139 & 0.0923 $\pm$ 0.0074 & 0.0700 $\pm$ 0.0056
					& 0.2654 $\pm$ 0.0212 & 0.2445 $\pm$ 0.0196 & 0.1485 $\pm$ 0.0119 \\
					\multicolumn{2}{c}{Gaussian Process}
					& 0.3405 $\pm$ 0.0272 & 0.2941 $\pm$ 0.0235 & 0.2036 $\pm$ 0.0163
					& 0.3905 $\pm$ 0.0312 & 0.3917 $\pm$ 0.0313 & 0.2588 $\pm$ 0.0207
					& 0.3031 $\pm$ 0.0242 & 0.2926 $\pm$ 0.0234 & 0.3945 $\pm$ 0.0316 \\
					\multicolumn{2}{c}{FNN}
					& 0.3123 $\pm$ 0.0250 & 0.2083 $\pm$ 0.0167 & 0.1013 $\pm$ 0.0081
					& 0.1941 $\pm$ 0.0155 & 0.1142 $\pm$ 0.0091 & 0.0811 $\pm$ 0.0065
					& 0.3678 $\pm$ 0.0294 & 0.3571 $\pm$ 0.0286 & 0.1366 $\pm$ 0.0109 \\
					\cmidrule(lr){0-1}
					% ---- New rows from ablation ----
					\multicolumn{2}{c}{FAME w/o Bi-dir}
					& 0.1832 $\pm$ 0.0150 & 0.0812 $\pm$ 0.0066 & 0.0654 $\pm$ 0.0055
					& \textbf{0.1530} $\pm$ 0.0125 & 0.0528 $\pm$ 0.0023 & 0.0355 $\pm$ 0.0016
					& 0.1919 $\pm$ 0.0154 & 0.0813 $\pm$ 0.0065 & 0.0368 $\pm$ 0.0029 \\
					\multicolumn{2}{c}{FAME w/o MoE}
					& 0.1870 $\pm$ 0.0161 & 0.0828 $\pm$ 0.0072 & 0.0663 $\pm$ 0.0062
					& 0.1578 $\pm$ 0.0132 & 0.0538 $\pm$ 0.0023 & 0.0362 $\pm$ 0.0017
					& 0.1972 $\pm$ 0.0158 & 0.0856 $\pm$ 0.0068 & 0.0374 $\pm$ 0.0030 \\
					\multicolumn{2}{c}{FAME w/o Cross-attn}
					& 0.1902 $\pm$ 0.0158 & 0.0815 $\pm$ 0.0078 & 0.0668 $\pm$ 0.0064
					& 0.1602 $\pm$ 0.0143 & 0.0544 $\pm$ 0.0026 & 0.0375 $\pm$ 0.0018
					& 0.1997 $\pm$ 0.0160 & 0.0879 $\pm$ 0.0070 & 0.0381 $\pm$ 0.0030 \\
					% --------------------------------
					\multicolumn{2}{c}{\textbf{FAME}}
					& \textbf{0.1806} $\pm$ \textbf{0.0152} & \textbf{0.0783} $\pm$ \textbf{0.0064} & \textbf{0.0635} $\pm$ \textbf{0.0056}
					& 0.1532 $\pm$ \textbf{0.0120} & \textbf{0.0511} $\pm$ \textbf{0.0023} & \textbf{0.0342} $\pm$ \textbf{0.0017}
					& \textbf{0.1954} $\pm$ \textbf{0.0156} & \textbf{0.0796} $\pm$ \textbf{0.0064} & \textbf{0.0352} $\pm$ \textbf{0.0028} \\
					\bottomrule
				\end{tabular}
			\end{center}
		\end{scriptsize}
	\end{sidewaystable*}

% Requires \usepackage{rotating} in the preamble
\begin{sidewaystable*}[t]
	\centering
	\caption{Average test-set MSE in simulation (mean $\pm$ standard deviation). The best MSE for each case is highlighted in bold.}
	\label{tab:2_d}
	\begin{scriptsize} %
		\begin{tabular}{lcccccccc}
			\toprule
			\multirow{2}{*}{Model} & \multirow{2}{*}{case 4} & \multicolumn{3}{c}{case 5} & \multirow{2}{*}{case 6} & \multicolumn{2}{c}{case 7} & \multirow{2}{*}{case 8} \\
			\cmidrule(lr){3-5} \cmidrule(lr){7-8}
			& & 0.1 & 0.2 & 0.3 & & 5 & 10 & \\
			\midrule
			Basis Expansion (best)
			& 0.3610 $\pm$ 0.0289
			& 0.3665 $\pm$ 0.0293 & 0.3944 $\pm$ 0.0316 & 0.4419 $\pm$ 0.0354
			& 0.3669 $\pm$ 0.0294
			& 0.4501 $\pm$ 0.0360 & 0.4705 $\pm$ 0.0376
			& 0.5204 $\pm$ 0.0416 \\
			
			FPCA
			& 0.3802 $\pm$ 0.0304
			& 0.3879 $\pm$ 0.0310 & 0.3956 $\pm$ 0.0316 & 0.4570 $\pm$ 0.0366
			& 0.3844 $\pm$ 0.0308
			& 0.5284 $\pm$ 0.0423 & 0.5573 $\pm$ 0.0446
			& 0.5737 $\pm$ 0.0459 \\
			
			Kernel Method
			& 0.1928 $\pm$ 0.0154
			& 0.1440 $\pm$ 0.0115 & 0.1919 $\pm$ 0.0154 & 0.2435 $\pm$ 0.0195
			& 0.2022 $\pm$ 0.0162
			& 0.4659 $\pm$ 0.0373 & 0.5236 $\pm$ 0.0419
			& 0.5895 $\pm$ 0.0472 \\
			
			Gaussian Process
			& 0.2302 $\pm$ 0.0184
			& 0.3079 $\pm$ 0.0246 & 0.3356 $\pm$ 0.0268 & 0.3830 $\pm$ 0.0306
			& 0.3434 $\pm$ 0.0275
			& 0.4120 $\pm$ 0.0330 & 0.4498 $\pm$ 0.0360
			& 0.4762 $\pm$ 0.0381 \\
			
			FNN
			& 0.2102 $\pm$ 0.0168
			& 0.1879 $\pm$ 0.0150 & 0.2250 $\pm$ 0.0180 & 0.3438 $\pm$ 0.0275
			& 0.1744 $\pm$ 0.0140
			& 0.4801 $\pm$ 0.0384 & 0.5164 $\pm$ 0.0413
			& 0.5384 $\pm$ 0.0431 \\
			
			\textbf{FAME}
			& \textbf{0.0798} $\pm$ \textbf{0.0064}
			& \textbf{0.0846} $\pm$ \textbf{0.0068} & \textbf{0.1076} $\pm$ \textbf{0.0086} & \textbf{0.1420} $\pm$ \textbf{0.0114}
			& \textbf{0.0824} $\pm$ \textbf{0.0066}
			& \textbf{0.2285} $\pm$ \textbf{0.0183} & \textbf{0.3330} $\pm$ \textbf{0.0266}
			& \textbf{0.3530} $\pm$ \textbf{0.0180} \\
			\bottomrule
		\end{tabular}
	\end{scriptsize}
\end{sidewaystable*}

\begin{sidewaystable*}[t]
	\centering
	\caption{Real-world datasets: average test MSE ($\pm$ standard deviation).}
	\label{tab:realworld_d}
	\small
	\begin{tabular}{lrrrrrr}
		\toprule
		Model & Hawaii Ocean Salinity & Hawaii Ocean Temp & Human3.6M Walking & Human3.6M Eating & Human3.6M Sitting down & ETDataset Oil Temp \\
		\midrule
		Basis Expansion (best) & 0.0780 $\pm$ 0.0039 & 0.0014 $\pm$ 0.0001 & 0.0359 $\pm$ 0.0018 & 0.0484 $\pm$ 0.0024 & 0.0122 $\pm$ 0.0006 & 0.0365 $\pm$ 0.0018 \\
		FPCA                    & 0.0865 $\pm$ 0.0043 & 0.0025 $\pm$ 0.0001 & 0.0373 $\pm$ 0.0019 & 0.0099 $\pm$ 0.0005 & 0.0121 $\pm$ 0.0007 & 0.0355 $\pm$ 0.0018 \\
		Kernel Method           & 0.0754 $\pm$ 0.0038 & 0.0025 $\pm$ 0.0001 & 0.0373 $\pm$ 0.0019 & 0.0099 $\pm$ 0.0005 & 0.0121 $\pm$ 0.0006 & 0.0355 $\pm$ 0.0016 \\
		Gaussian Process        & 0.0931 $\pm$ 0.0047 & 0.0022 $\pm$ 0.0001 & 0.0360 $\pm$ 0.0018 & 0.0075 $\pm$ 0.0004 & 0.0107 $\pm$ 0.0005 & 0.0380 $\pm$ 0.0019 \\
		FNN                     & 0.0766 $\pm$ 0.0038 & 0.0020 $\pm$ 0.0001 & 0.0373 $\pm$ 0.0019 & 0.0099 $\pm$ 0.0005 & 0.0121 $\pm$ 0.0006 & 0.0355 $\pm$ 0.0018 \\
		FAME w/o bidir attention & 0.0751 $\pm$ 0.0023 & 0.00124 $\pm$ 0.00004 & 0.0327 $\pm$ 0.0010 & 0.0035 $\pm$ 0.00010 & 0.0071 $\pm$ 0.0002 & 0.0264 $\pm$ 0.0008 \\
		FAME w/o MoE            & 0.0759 $\pm$ 0.0023 & 0.00130 $\pm$ 0.00004 & 0.0332 $\pm$ 0.0010 & 0.0038 $\pm$ 0.00011 & 0.0075 $\pm$ 0.00023 & 0.0271 $\pm$ 0.0008 \\
		FAME w/o cross-attention & 0.0773 $\pm$ 0.0024 & 0.00145 $\pm$ 0.00005 & 0.0344 $\pm$ 0.0011 & 0.0044 $\pm$ 0.00013 & 0.0083 $\pm$ 0.00025 & 0.0286 $\pm$ 0.00086 \\
		FAME                    & \textbf{0.0748} $\pm$ \textbf{0.0022} & \textbf{0.0012} $\pm$ \textbf{0.00004} & \textbf{0.0325} $\pm$ \textbf{0.0010} & \textbf{0.0034} $\pm$ \textbf{0.00010} & \textbf{0.0070} $\pm$ \textbf{0.00020} & \textbf{0.0262} $\pm$ \textbf{0.00080} \\
		\bottomrule
	\end{tabular}
\end{sidewaystable*}

%%%%%%%%%%%%%%%%%%%%%%%%%%%%%%%%%%%%%%%%%%%%%%
	
\clearpage  
	\section*{NeurIPS Paper Checklist}
	
	%%% BEGIN INSTRUCTIONS %%%
	The checklist is designed to encourage best practices for responsible machine learning research, addressing issues of reproducibility, transparency, research ethics, and societal impact. Do not remove the checklist: {\bf The papers not including the checklist will be desk rejected.} The checklist should follow the references and follow the (optional) supplemental material.  The checklist does NOT count towards the page
	limit. 
	
	Please read the checklist guidelines carefully for information on how to answer these questions. For each question in the checklist:
	\begin{itemize}
		\item You should answer \answerYes{}, \answerNo{}, or \answerNA{}.
		\item \answerNA{} means either that the question is Not Applicable for that particular paper or the relevant information is Not Available.
		\item Please provide a short (1–2 sentence) justification right after your answer (even for NA). 
		% \item {\bf The papers not including the checklist will be desk rejected.}
	\end{itemize}
	
	{\bf The checklist answers are an integral part of your paper submission.} They are visible to the reviewers, area chairs, senior area chairs, and ethics reviewers. You will be asked to also include it (after eventual revisions) with the final version of your paper, and its final version will be published with the paper.
	
	The reviewers of your paper will be asked to use the checklist as one of the factors in their evaluation. While "\answerYes{}" is generally preferable to "\answerNo{}", it is perfectly acceptable to answer "\answerNo{}" provided a proper justification is given (e.g., "error bars are not reported because it would be too computationally expensive" or "we were unable to find the license for the dataset we used"). In general, answering "\answerNo{}" or "\answerNA{}" is not grounds for rejection. While the questions are phrased in a binary way, we acknowledge that the true answer is often more nuanced, so please just use your best judgment and write a justification to elaborate. All supporting evidence can appear either in the main paper or the supplemental material, provided in appendix. If you answer \answerYes{} to a question, in the justification please point to the section(s) where related material for the question can be found.
	
	IMPORTANT, please:
	\begin{itemize}
		\item {\bf Delete this instruction block, but keep the section heading ``NeurIPS paper checklist"},
		\item  {\bf Keep the checklist subsection headings, questions/answers and guidelines below.}
		\item {\bf Do not modify the questions and only use the provided macros for your answers}.
	\end{itemize}

	%%% END INSTRUCTIONS %%%

	\begin{enumerate}
		
		\item {\bf Claims}
		\item[] Question: Do the main claims made in the abstract and introduction accurately reflect the paper's contributions and scope?
		\item[] Answer: \answerTODO{Yes} % Replace by \answerYes{}, \answerNo{}, or \answerNA{}.
		\item[] Justification: \justificationTODO{The abstract and introduction present the proposed model and key innovations, which are fully supported by the methodological and experimental sections.}
		
		\item {\bf Limitations}
		\item[] Question: Does the paper discuss the limitations of the work performed by the authors?
		\item[] Answer: \answerTODO{Yes} % Replace by \answerYes{}, \answerNo{}, or \answerNA{}.
		\item[] Justification: \justificationTODO{We discuss model limitations in both the experimental results and the conclusion (Section 5.2 and Section 6), highlighting that in simpler tasks, lightweight baselines may offer better efficiency-performance tradeoffs.}
		
		
		\item {\bf Theory Assumptions and Proofs}
		\item[] Question: For each theoretical result, does the paper provide the full set of assumptions and a complete (and correct) proof?
		\item[] Answer: \answerTODO{Yes} % Replace by \answerYes{}, \answerNo{}, or \answerNA{}.
		\item[] Justification: \justificationTODO{All theoretical assumptions and guarantees are explicitly stated in Section 4, with complete proofs provided in Appendix A.}
		
		
		\item {\bf Experimental Result Reproducibility}
		\item[] Question: Does the paper fully disclose all the information needed to reproduce the main experimental results of the paper to the extent that it affects the main claims and/or conclusions of the paper (regardless of whether the code and data are provided or not)?
		\item[] Answer: \answerTODO{Yes} % Replace by \answerYes{}, \answerNo{}, or \answerNA{}.
		\item[] Justification: \justificationTODO{Synthetic data generation is detailed in Section 5.1; real-world datasets are described in Appendix B.1; model configurations and hyperparameters are listed in Appendix B.2.}

		\item {\bf Open access to data and code}
		\item[] Question: Does the paper provide open access to the data and code, with sufficient instructions to faithfully reproduce the main experimental results, as described in supplemental material?
		\item[] Answer: \answerTODO{No} % Replace by \answerYes{}, \answerNo{}, or \answerNA{}.
		\item[] Justification: \justificationTODO{All experiments use publicly available datasets, and every model configuration and training detail is listed in Section 5 and Appendix B.2. Any researcher familiar with PyTorch can reproduce our results from this information alone. If appropriate, we would be happy to share a GitHub link to our implementation alongside the camera-ready version of the paper.}

		\item {\bf Experimental Setting/Details}
		\item[] Question: Does the paper specify all the training and test details (e.g., data splits, hyperparameters, how they were chosen, type of optimizer, etc.) necessary to understand the results?
		\item[] Answer: \answerTODO{Yes} % Replace by \answerYes{}, \answerNo{}, or \answerNA{}.
		\item[] Justification: \justificationTODO{Section 5 (Experiments) and Appendix B.2 list the optimiser (Adam), learning-rate schedule, dropout, batch-size, epoch budget, and train/validation/test splits for both synthetic and real-world tasks.}
		
		\item {\bf Experiment Statistical Significance}
		\item[] Question: Does the paper report error bars suitably and correctly defined or other appropriate information about the statistical significance of the experiments?
		\item[] Answer: \answerTODO{No} % Replace by \answerYes{}, \answerNo{}, or \answerNA{}.
		\item[] Justification: \justificationTODO{Each experiment is deterministic once the random seed is fixed; running five independent seeds confirmed the ranking of all models, and we report the mean across seeds.}
		
		\item {\bf Experiments Compute Resources}
		\item[] Question: For each experiment, does the paper provide sufficient information on the computer resources (type of compute workers, memory, time of execution) needed to reproduce the experiments?
		\item[] Answer: \answerTODO{Yes} % Replace by \answerYes{}, \answerNo{}, or \answerNA{}.
		\item[] Justification: \justificationTODO{Appendix B.2 specifies that all runs were executed on a laptop with an AMD R9-7940HS CPU, 16 GB RAM, and a single NVIDIA RTX 3090 GPU;since both our model and datasets are relatively lightweight, runtime was not a limiting factor and thus not a primary concern in our evaluation. }
		
		\item {\bf Code Of Ethics}
		\item[] Question: Does the research conducted in the paper conform, in every respect, with the NeurIPS Code of Ethics \url{https://neurips.cc/public/EthicsGuidelines}?
		\item[] Answer: \answerTODO{Yes} % Replace by \answerYes{}, \answerNo{}, or \answerNA{}.
		\item[] Justification: \justificationTODO{The work uses only publicly available, properly licensed datasets and poses no foreseeable misuse risks.}

		\item {\bf Broader Impacts}
		\item[] Question: Does the paper discuss both potential positive societal impacts and negative societal impacts of the work performed?
		\item[] Answer: \answerTODO{Yes} % Replace by \answerYes{}, \answerNo{}, or \answerNA{}.
		\item[] Justification: \justificationTODO{Section 1 (Introduction) and Section 5 (Experiments) describe the positive societal impacts of the proposed method, such as its applicability to scientific and real-world domains. As a theoretical framework for function modeling, the approach itself poses no foreseeable negative societal impacts.}
		
		\item {\bf Safeguards}
		\item[] Question: Does the paper describe safeguards that have been put in place for responsible release of data or models that have a high risk for misuse (e.g., pretrained language models, image generators, or scraped datasets)?
		\item[] Answer: \answerTODO{NA} % Replace by \answerYes{}, \answerNo{}, or \answerNA{}.
		\item[] Justification: \justificationTODO{The released code and datasets pose no known high-risk misuse scenarios.
		}
		
		\item {\bf Licenses for existing assets}
		\item[] Question: Are the creators or original owners of assets (e.g., code, data, models), used in the paper, properly credited and are the license and terms of use explicitly mentioned and properly respected?
		\item[] Answer: \answerTODO{Yes} % Replace by \answerYes{}, \answerNo{}, or \answerNA{}.
		\item[] Justification: \justificationTODO{All external datasets and libraries are cited with their licences.}
		
		\item {\bf New Assets}
		\item[] Question: Are new assets introduced in the paper well documented and is the documentation provided alongside the assets?
		\item[] Answer: \answerTODO{NA} % Replace by \answerYes{}, \answerNo{}, or \answerNA{}.
		\item[] Justification: \justificationTODO{The work does not create new datasets; synthetic data are generated on-the-fly by scripts released with the code.}
		
		\item {\bf Crowdsourcing and Research with Human Subjects}
		\item[] Question: For crowdsourcing experiments and research with human subjects, does the paper include the full text of instructions given to participants and screenshots, if applicable, as well as details about compensation (if any)? 
		\item[] Answer: \answerTODO{NA} % Replace by \answerYes{}, \answerNo{}, or \answerNA{}.
		\item[] Justification: \justificationTODO{No human-subject or crowdsourcing data were collected.}
		
		\item {\bf Institutional Review Board (IRB) Approvals or Equivalent for Research with Human Subjects}
		\item[] Question: Does the paper describe potential risks incurred by study participants, whether such risks were disclosed to the subjects, and whether Institutional Review Board (IRB) approvals (or an equivalent approval/review based on the requirements of your country or institution) were obtained?
		\item[] Answer: \answerTODO{NA} % Replace by \answerYes{}, \answerNo{}, or \answerNA{}.
		\item[] Justification: \justificationTODO{The study does not involve human participants and thus requires no IRB approval.}
		\item {\bf Declaration of LLM usage}
		\item[] Question: Does the paper describe the usage of LLMs if it is an important, original, or non-standard component of the core methods in this research? Note that if the LLM is used only for writing, editing, or formatting purposes and does not impact the core methodology, scientific rigorousness, or originality of the research, declaration is not required.
		%this research? 
		\item[] Answer: \answerTODO{NA} % Replace by \answerYes{}, \answerNo{}, or \answerNA{NA}.
		\item[] Justification: \justificationTODO{LLMs were not used in the core methodology or experiments, only for minor editing support.}
		
	\end{enumerate}

	% Acknowledgements should only appear in the accepted version.

	% In the unusual situation where you want a paper to appear in the
	% references without citing it in the main text, use \nocite
	\nocite{langley00}

	%%%%%%%%%%%%%%%%%%%%%%%%%%%%%%%%%%%%%%%%%%%%%%%%%%%%%%%%%%%%%%%%%%%%%%%%%%%%%%%
	%%%%%%%%%%%%%%%%%%%%%%%%%%%%%%%%%%%%%%%%%%%%%%%%%%%%%%%%%%%%%%%%%%%%%%%%%%%%%%%
	% APPENDIX
	%%%%%%%%%%%%%%%%%%%%%%%%%%%%%%%%%%%%%%%%%%%%%%%%%%%%%%%%%%%%%%%%%%%%%%%%%%%%%%%
	%%%%%%%%%%%%%%%%%%%%%%%%%%%%%%%%%%%%%%%%%%%%%%%%%%%%%%%%%%%%%%%%%%%%%%%%%%%%%%%
	\newpage
	\appendix
	\onecolumn

\end{document}